\documentclass{article}

 \usepackage[preprint]{neurips_2026}

\usepackage[utf8]{inputenc} 
\usepackage[T1]{fontenc}    
\usepackage[table]{xcolor}
\definecolor{dark-blue}{RGB}{0,0,191}
\usepackage[hidelinks,colorlinks=true,citecolor=dark-blue,linkcolor=dark-blue,urlcolor=dark-blue, backref=page]{hyperref}
\usepackage{url}            
\usepackage{booktabs}       
\usepackage{amsfonts}       
\usepackage{nicefrac}       
\usepackage{microtype}      
\usepackage{xcolor}         

\title{Pessimism-Free Offline Learning in General-Sum Games via KL Regularization}

\usepackage{pifont}

\usepackage{amsmath}
\usepackage{amssymb}
\usepackage{mathtools}
\usepackage{amsthm}
\usepackage{color}
\mathtoolsset{showonlyrefs}

\usepackage{algorithm}
\usepackage{algpseudocode}

\usepackage{enumitem}

\usepackage[capitalize,noabbrev]{cleveref}

\usepackage{xurl} 
\usepackage{listings}
\usepackage{xcolor}
\usepackage{thmtools}

\usepackage{pifont}
\theoremstyle{plain}
\newtheorem{theorem}{Theorem}[section]

\newtheorem{lemma}[theorem]{Lemma}

\theoremstyle{definition}

\newtheorem{assumption}[theorem]{Assumption}
\theoremstyle{remark}

\usepackage{algorithm}
\usepackage{algorithmicx}
\usepackage{algpseudocode}

\DeclareMathOperator*{\argmin}{argmin}
\DeclareMathOperator*{\argmax}{argmax}

\usepackage[textsize=tiny]{todonotes}

\author{%
  Claire Chen\\
  The Division of Physics,
Mathematics and Astronomy \\
California Institute of Technology\\
  \texttt{clairechen@caltech.edu} \\
    \And
    Yuheng Zhang \\
  Department of Computer Science \\
  University of Illinois Urbana-Champaign \\
  \texttt{ yuhengz2@illinois.edu} \\
}

\begin{document}

\maketitle

\begin{abstract}

Offline multi-agent reinforcement learning in general-sum settings is challenged by the distribution shift between logged datasets and target equilibrium policies. While standard methods rely on manual pessimistic penalties, we demonstrate that KL regularization suffices to stabilize learning and achieve equilibrium recovery. We propose General-sum Anchored Nash Equilibrium (GANE), which recovers regularized Nash equilibria at an accelerated statistical rate of $\tilde{\mathcal{O}}(1/n)$. For computational tractability, we develop General-sum Anchored Mirror Descent (GAMD), an iterative algorithm converging to a Coarse Correlated Equilibrium at the standard rate of $\tilde{\mathcal{O}}(1/\sqrt{n} + 1/T)$. These results establish KL regularization as a standalone mechanism for pessimism-free offline learning that achieves equivalent or accelerated rates in multi-player general-sum games.


\end{abstract}

\section{Introduction}
\label{sec:intro}
Offline reinforcement learning (RL) provides a framework for developing decision-making policies from pre-collected datasets, making it suitable for applications where active exploration is restricted by safety or financial costs \citep{levine2020offline}. While much of the offline literature has historically addressed single-agent optimization, many strategic environments involve interactions among multiple agents with distinct objectives. General-sum games offer a flexible model for these scenarios, capturing strategic complexities in economic markets, multi-agent coordination, and the alignment of large language models through collective human preferences \citep{ouyang2022training, ye2024online}. In this offline multi-agent setting, the learning objective is to recover an equilibrium policy using only fixed historical data.

A central difficulty in offline general-sum games is distribution shift, which occurs when the learned strategy explores regions of the state-action space that are poorly represented in the logged dataset \citep{fujimoto2019off, zhang2023offline}. Standard methodologies typically address this shift through the principle of explicit pessimism \citep{jin2021pessimism, zhang2023offline}. These approaches incorporate manually designed lower confidence bounds (LCB) or penalty terms to suppress the estimated values of actions that lack sufficient offline data support \citep{cui2022offline, zhong2022pessimistic}. While theoretically sound, these pessimistic mechanisms often require the construction of complex uncertainty quantifiers over joint action spaces and intensive hyperparameter tuning to balance conservatism with performance \citep{cui2022offline, zhang2023offline}.

Recently, KL regularization with respect to a fixed reference policy has become a common objective for stabilizing multi-agent learning and enforcing behavioral constraints \citep{ouyang2022training, ye2024online}. While this framework has gained significant prominence in single-agent settings—for instance, in language model alignment through Reinforcement Learning from Human Feedback (RLHF) \citep{ouyang2022training}—its application to multi-player general-sum settings remains relatively scarce. In existing literature that utilizes KL regularization within game-theoretic contexts, the regularization is typically treated solely as an anchor to a reference distribution, while distribution shift is handled by another separate, explicit pessimism mechanism such as lower confidence bounds \citep{ye2024online}.However, the capacity of KL regularization to act as a standalone, implicit pessimism mechanism that independently stabilizes learning against distribution shift in multi-player general-sum games remains unestablished.

In this work, we establish that KL regularization serves as a standalone, pessimism-free mechanism to stabilize offline learning in general-sum games. To achieve this, we leverage the reference-anchored data coverage framework \citep{chen2026fast}, which bypasses the conceptual circularity inherent in prior works by anchoring coverage requirements to a known reference policy rather than an unknown equilibrium. We first introduce General-sum Anchored Nash Equilibrium (GANE), a theoretical framework that recovers regularized Nash equilibria at an accelerated statistical rate of $\tilde{\mathcal{O}}(1/n)$. Our analysis reveals that the independent product structure of Nash policies enables the exact cancellation of first-order estimation errors, allowing GANE to bypass the standard $\tilde{\mathcal{O}}(1/\sqrt{n})$ bottleneck encountered in unregularized offline games \citep{zhang2023offline, cui2022offline}. Recognizing that computing a Nash equilibrium is often computationally intractable, we further develop General-sum Anchored Mirror Descent (GAMD), a decentralized iterative algorithm. We prove that GAMD recovers a Coarse Correlated Equilibrium at the standard rate of $\tilde{\mathcal{O}}(1/\sqrt{n} + 1/T)$ for $n$ samples and $T$ iterations. Our primary contributions are summarized as follows:


\begin{enumerate}
    \item \textbf{Pessimism-Free Equilibrium Recovery}: We establish that KL regularization serves as an effective alternative to explicit pessimism, allowing for equilibrium recovery in general-sum settings without the overhead of tuning confidence-based bonuses.
    \item \textbf{Accelerated Rates for Nash Equilibria}: We prove that GANE achieves an $\tilde{\mathcal{O}}(1/n)$ statistical rate. This result demonstrates that the product structure of Nash equilibria can be leveraged to achieve higher statistical efficiency than previously established for unregularized general-sum offline learning.
    \item \textbf{Tractable Iterative Learning}: We show that GAMD provides a computationally efficient path to recovering Coarse Correlated Equilibria, achieving the standard $\tilde{\mathcal{O}}(1/\sqrt{n})$ statistical rate through simple, decentralized updates without relying on explicit pessimism.
\end{enumerate}

\paragraph{Organization.}
The remainder of this paper is organized as follows. In Section~\ref{sec:background}, we formalize the $m$-player general-sum game, define the KL-regularized value functions, and establish the offline learning model. Section~\ref{sec:coverage} reviews the landscape of concentrability assumptions in multi-agent systems and the reference-anchored coverage requirements. Section~\ref{sec:idealized_algorithm_ne} introduces the GANE algorithm and presents the accelerated $\tilde{\mathcal{O}}(1/n)$ statistical rate for Nash equilibria. In Section~\ref{sec:practical_algorithm}, we propose the iterative GAMD algorithm and characterize its optimization and sample complexity guarantees for recovering Coarse Correlated Equilibria. We review related work in Section~\ref{sec:related_work}, provide a broader discussion on future directions in Section~\ref{sec:discussion}, and conclude the paper in Section~\ref{sec:conclusion}.


\begin{table}[h]
\label{tab:comparison_general_sum}
\centering
\small
\begin{tabular}{lcccc}
\toprule
\textbf{Paper} & \textbf{NE Rate} & \textbf{CCE Rate} & \textbf{Optimization ($T$)} & \textbf{Pessimism} \\ \midrule
\citet{cui2022provably} & $\tilde{\mathcal{O}}(1/\sqrt{n})$ & --- & Oracle & Yes \\
\citet{zhang2023offline} & $\tilde{\mathcal{O}}(1/\sqrt{n})$ & $\tilde{\mathcal{O}}(1/\sqrt{n})$ & Oracle & Yes \\
\rowcolor{gray!10} \textbf{This Work} & $\mathbf{\tilde{\mathcal{O}}(1/n)}$ & $\mathbf{\tilde{\mathcal{O}}(1/\sqrt{n})}$ & $\mathbf{\tilde{\mathcal{O}}(1/T)}$ & \textbf{No} \\ \bottomrule
\end{tabular}
\caption{Comparison of offline learning guarantees in multi-player general-sum games. We report statistical rates in terms of samples $n$ and optimization convergence in terms of iterations $T$. Our work is the first to establish a fast $\tilde{\mathcal{O}}(1/n)$ rate for Nash Equilibria and a  computationally tractable solution for CCE without explicit pessimism. The ``Oracle'' labels denote a reliance on either black-box stage-game equilibrium solvers \citep{cui2022provably} or global optimization oracles over general function classes \citep{zhang2023offline}.}
\end{table}

\section{Problem Formulation}
\label{sec:background}

We consider an $m$-player general-sum contextual bandit game defined by the tuple $\mathcal{M} = (\mathcal{N}, \mathcal{X}, \{\mathcal{A}_i\}_{i=1}^m, \{r_i^\star\}_{i=1}^m, \rho)$. Here, $\mathcal{N} = [m] \coloneqq \{1, \dots, m\}$ denotes the set of players, $\mathcal{X}$ is the context space, and $\mathcal{A}_i$ is the finite action space for player $i$. In each round, a context $x \sim \rho$ is sampled. All players simultaneously select actions $a_i \in \mathcal{A}_i$, forming a joint action profile $\boldsymbol{a} = (a_1, \dots, a_m) \in \mathcal{A} \coloneqq \prod_{i=1}^m \mathcal{A}_i$. Each player $i$ receives a deterministic reward $r_i^\star(x, \boldsymbol{a}) \in [0, 1]$.

A joint policy $\pi: \mathcal{X} \to \Delta(\mathcal{A})$ maps contexts to a distribution over joint actions. We denote the marginalized policy for player $i$ as $\pi_i(\cdot \mid x)$ and the joint policy of all other players as $\pi_{-i}(\cdot \mid x)$.

\subsection{KL-Regularized Objectives}
To address distribution shift without explicit pessimism, we utilize KL regularization relative to fixed reference policies $\pi^{\mathrm{ref}} = (\pi^{\mathrm{ref}}_1, \dots, \pi^{\mathrm{ref}}_m)$. For a regularization parameter $\eta > 0$, the Q-function for player $i$ is 
\begin{align}
    Q_i^{\pi}(x, \boldsymbol{a}) \coloneqq r_i^\star(x, \boldsymbol{a}).
\end{align}The KL-regularized value function $V_i^\pi(x)$ is defined as:
\begin{equation}
    V_i^{\pi}(x) = \mathbb{E}_{\boldsymbol{a} \sim \pi(\cdot|x)} [Q_i^{\pi}(x, \boldsymbol{a})] - \eta^{-1} \mathrm{KL}\big(\pi_i(\cdot|x) \,\|\, \pi^{\mathrm{ref}}_i(\cdot|x)\big).
    \label{eq:value_kl_def}
\end{equation}
Each player $i$ seeks to maximize their own regularized value function. We define the best-response value against an opponent profile $\pi_{-i}$ as 
\begin{align}
V_i^{\dagger, \pi_{-i}}(x) \coloneqq \max_{\pi'_i} V_i^{\pi'_i, \pi_{-i}}(x).    
\end{align}
We also denote $\Lambda_{\eta, m}\coloneqq \exp(\eta m)$ as the joint anchoring constant measuring the density shift relative to the reference policy.


\subsection{Target Equilibria}
\label{subsec:equilibria}

We analyze two equilibrium concepts in the regularized $m$-player game. Conceptually, an equilibrium represents a stable state where no player can unilaterally improve their own regularized value by changing their strategy, assuming all other players' strategies remain fixed. The primary distinction between these concepts lies in the structural constraints placed on the joint policy and the resulting ability for players to coordinate.

\paragraph{Regularized Nash Equilibrium (NE).} A joint policy $\pi^\star$ is a regularized NE if it is a \textbf{product policy}, factorizing as
$\pi^\star(\boldsymbol{a} \mid x) = \prod_{i=1}^m \pi_i^\star(a_i \mid x)$. 
This captures scenarios where players act completely independently. The equilibrium condition implies that no player $i$ can increase their value by switching to a different independent strategy $\pi_i'$, given that opponents follow their respective marginal strategies $\pi_{-i}^\star$. For a learned product policy $\hat{\pi} = \prod_i \hat{\pi}_i$, we evaluate the NE Total Exploitability:
\begin{equation}
    \mathrm{Gap}_{\mathrm{NE}}(\hat{\pi}) \coloneqq \sum_{i=1}^m \mathbb{E}_{x \sim \rho} \left[ V_i^{\dagger, \hat{\pi}_{-i}}(x) - V_i^{\hat{\pi}_1 \times \dots \times \hat{\pi}_m}(x) \right].
\end{equation}

\paragraph{Regularized Coarse Correlated Equilibrium (CCE).} A regularized CCE is a broader concept that allows the joint policy $\pi^\star$ to be \textbf{any distribution} in the simplex $\Delta(\mathcal{A})$, potentially incorporating correlations between players (e.g., through a shared latent signal). The stability condition here implies that a player cannot improve their expected return by choosing to deviate from the suggested joint distribution before observing their specific action. While every Nash Equilibrium is a CCE, a CCE allows for complex coordination that is forbidden in the NE setting. For a learned (possibly correlated) joint policy $\bar{\pi}$, we evaluate the CCE Total Exploitability:
\begin{equation}
    \mathrm{Gap}_{\mathrm{CCE}}(\bar{\pi}) \coloneqq \sum_{i=1}^m \mathbb{E}_{x \sim \rho} \left[ V_i^{\dagger, \bar{\pi}_{-i}}(x) - V_i^{\bar{\pi}}(x) \right].
\end{equation}


While we focus on NE and CCE in this work, we consider the extension to Correlated Equilibrium (CE) a promising direction for future research. Our current analysis serves as a foundational demonstration of how KL regularization alone—without the need for explicit pessimism—suffices to establish both a fast $\tilde{\mathcal{O}}(1/n)$ rate for Nash equilibria and a computationally tractable solution for CCE. By characterizing these two distinct regimes, we lay the groundwork for a broader regularized framework across the entire spectrum of game-theoretic equilibria.

\subsection{Offline Learning Model}
\label{subsec:offline_model}

The learner has access to a static dataset $\mathcal{D} = \{ (x_{\tau}, \boldsymbol{a}_{\tau}, \mathbf{r}_{\tau}) \}_{\tau=1}^{n}$ generated by an unknown behavioral distribution $\mu(x, \boldsymbol{a})$. For each interaction $\tau$, the observed reward vector $\mathbf{r}_{\tau} = (r_{\tau,1}, \dots, r_{\tau,m})$ consists of noisy realizations of the true rewards: $r_{\tau,i} = r_i^\star(x_{\tau}, \boldsymbol{a}_{\tau}) + \xi_{\tau,i}$, where $\xi_{\tau,i}$ is independent zero-mean $1$-sub-Gaussian noise. We assume access to function classes $\{\mathcal{Q}_i\}_{i=1}^m$ for reward estimation, where each $\mathcal{Q}_i$ is a finite class of functions mapping $\mathcal{X} \times \mathcal{A} \to [0, 1]$. Following standard literature \citep{xie2021batch, zhang2026beyond}, we assume \textit{realizability}, i.e., $r_i^\star \in \mathcal{Q}_i$ for all $i \in [m]$. For any estimate $\hat{Q}_i \in \mathcal{Q}_i$, the pointwise regression error is defined as $\mathcal{Z}_i(x, \boldsymbol{a}) \coloneqq \hat{Q}_i(x, \boldsymbol{a}) - r_i^\star(x, \boldsymbol{a})$.



\section{Unilateral Data Coverage}
\label{sec:coverage}

Establishing sample-efficient recovery in offline multi-agent systems requires that the historical dataset $\mathcal{D}$ provides sufficient coverage of the state-action regions relevant to the target objective. In this section, we contrast the traditional unilateral concentrability assumption with the reference-anchored framework used in this work.

\subsection{Standard Unilateral Concentrability and the Circularity Problem}

The principle of \textit{unilateral concentrability} is often employed to bypass the exponential scaling of joint action spaces in multi-agent environments \citep{cui2022provably, cui2022offline}. This principle requires the dataset only to support scenarios where a single agent deviates from a target strategy while all others remain stationary.

\begin{assumption}[Classical Unilateral Concentrability \citep{cui2022provably}]
Let $\pi^* = (\pi_1^*, \dots, \pi_m^*)$ be a target equilibrium policy. Let $\Pi_{\mathrm{uni}}(\pi^*) \coloneqq \bigcup_{i=1}^m \left( \Pi_i \times \{\pi_{-i}^*\} \right)$ be the set of unilateral deviations from $\pi^*$. There exists a constant $C_{\mathrm{uni}}^* \ge 1$ such that for any $\pi' \in \Pi_{\mathrm{uni}}(\pi^*)$:
\begin{align}
    \left\| \frac{\rho(x)\pi'(\boldsymbol{a} \mid x)}{\mu(x, \boldsymbol{a})} \right\|_\infty \le C_{\mathrm{uni}}^*.
\end{align}
\end{assumption}

While this assumption provides favorable scaling ($\sum |\mathcal{A}_i|$ rather than $\prod |\mathcal{A}_i|$), it suffers from the \textbf{Curse of the Unknown Optimum}. Because $C_{\mathrm{uni}}^*$ is anchored to an unknown optimal equilibrium $\pi^*$, the assumption is conceptually circular: one must already know the target equilibrium to determine if the data is sufficient to find it. This makes the condition difficult to verify or estimate in practical strategic settings.

\subsection{The Reference-Anchored Unilateral Concentrability Framework}

To resolve the circularity problem, we utilize the \textbf{Reference-Anchored} coverage framework recently proposed in \citet{chen2026fast}. This condition shifts the theoretical burden from an unknown optimum to a fixed and known reference policy $\pi^{\mathrm{ref}}$.

\begin{assumption}[Reference-Anchored Unilateral Concentrability \citep{chen2026fast}]
\label{ass:concentrability}
Let $\pi^{\mathrm{ref}} = (\pi_1^{\mathrm{ref}}, \dots, \pi_m^{\mathrm{ref}})$ be the fixed reference joint policy. We define the set of \emph{reference-anchored unilateral deviation policies} $\Pi_{\mathrm{ref-uni}}$ as the union of profiles where one player deviates while all others adhere strictly to the reference policy: 
\begin{align}
    \Pi_{\mathrm{ref-uni}} \coloneqq \bigcup_{i=1}^m \left( \Pi_i \times \{\pi_{-i}^{\mathrm{ref}}\} \right).
\end{align}
Let $\mu$ denote the joint distribution of the offline dataset $\mathcal{D}$. We assume there exists a constant $C_{\mathrm{uni}} \ge 1$ such that for any $\pi' \in \Pi_{\mathrm{ref-uni}}$:
\begin{align}
    \left\| \frac{\rho(x)\pi'(\boldsymbol{a} \mid x)}{\mu(x, \boldsymbol{a})} \right\|_\infty \le C_{\mathrm{uni}}.
\end{align}
\end{assumption}

By anchoring to $\pi^{\mathrm{ref}}$, this framework provides two critical advantages: 
(i) Verifiability, as $\pi^{\mathrm{ref}}$ is known to the learner and $C_{\mathrm{uni}}$ can be empirically estimated; and (ii) Pessimism-Free Stability, as anchoring to the reference policy via KL regularization allows the algorithm to handle distribution shift without explicit bonuses.

\section{Pessimism-Free Nash Equilibria}
\label{sec:idealized_algorithm_ne}

Building on the regularized objective \eqref{eq:value_kl_def}, we investigate whether anchoring to a reference policy can independently mitigate the distribution shift inherent in offline multi-agent settings. We first introduce General-sum Anchored Nash Equilibrium (GANE), an algorithmic framework that leverages KL regularization to stabilize learning on the offline contextual bandit without the overhead of explicit pessimistic bonuses. GANE provides the foundational basis for our analysis, illustrating how the geometric properties of the regularizer interact with the structural properties of Nash equilibria.

GANE operates via empirical risk minimization over the static offline dataset. For every player $i \in [m]$, the algorithm first constructs an empirical estimate of the reward Q-function, $\hat{Q}_i$, by solving a regularized least-squares regression problem over the dataset $\mathcal{D}$ using the function class $\mathcal{Q}_i$. Given these empirical Q-functions, GANE assumes access to a stage-game equilibrium oracle that outputs a joint product policy $\hat{\pi}(\boldsymbol{a} \mid x) = \prod_{i=1}^m \hat{\pi}_i(a_i \mid x)$ forming an exact Regularized NE for the empirical game. Formally, $\hat{\pi}$ must satisfy the condition that no player $i$ can unilaterally improve their regularized value by deviating to any arbitrary policy $\pi_i'$, assuming the other players marginalize over $\hat{\pi}_{-i}$. 

\begin{algorithm}[t]
\caption{General-sum Anchored Nash Equilibrium (GANE)}
\label{alg:ideal_GANE}
\begin{algorithmic}[1]
    \State \textbf{Input:} Offline dataset $\mathcal{D}$, reference policies $\pi^{\mathrm{ref}} = (\pi_1^{\mathrm{ref}}, \dots, \pi_m^{\mathrm{ref}})$, regularization parameter $\eta > 0$, function classes $\{\mathcal{Q}_i\}_{i=1}^m$.
    \For{player $i = 1, \dots, m$}
        \State \textbf{Empirical Q-Update:} Estimate the Q-function via least-squares regression:
        \State $\hat{Q}_i \leftarrow \argmin_{f \in \mathcal{Q}_i} \sum_{\tau=1}^n \left( f(x_{\tau}, \boldsymbol{a}_{\tau}) - r_{\tau,i} \right)^2$.
    \EndFor
    \State \textbf{Equilibrium Computation:} For every context $x \in \mathcal{X}$, find a joint product policy $\hat{\pi}(\boldsymbol{a} \mid x) = \prod_{i=1}^m \hat{\pi}_i(a_i \mid x)$ that is a Regularized Nash Equilibrium of the empirical game $\{\hat{Q}_i(x, \cdot)\}_{i=1}^m$. Namely, for all $i \in [m]$ and all $\pi_i' \in \Delta(\mathcal{A}_i)$:
    \State 
    \vspace{-0.5em}
    \begin{equation*}
    \begin{aligned}
        &\mathbb{E}_{\boldsymbol{a} \sim \hat{\pi}(\cdot \mid x)}[\hat{Q}_i(x, \boldsymbol{a})] - \eta^{-1}\mathrm{KL}(\hat{\pi}_i(\cdot \mid x) \,\|\, \pi_i^{\mathrm{ref}}(\cdot \mid x)) \\
        &\quad \ge \mathbb{E}_{a_i \sim \pi_i'(\cdot \mid x), \boldsymbol{a}_{-i} \sim \hat{\pi}_{-i}(\cdot \mid x)}[\hat{Q}_i(x, \boldsymbol{a})] - \eta^{-1}\mathrm{KL}(\pi_i'(\cdot \mid x) \,\|\, \pi_i^{\mathrm{ref}}(\cdot \mid x)).
    \end{aligned}
    \end{equation*}
    \For{player $i = 1, \dots, m$}
        \State \textbf{Value Update:} Update the empirical value function for Player $i$:
        \State $\hat{V}_i(x) \leftarrow \mathbb{E}_{\boldsymbol{a} \sim \hat{\pi}(\cdot \mid x)}[\hat{Q}_i(x, \boldsymbol{a})] - \eta^{-1}\mathrm{KL}(\hat{\pi}_i(\cdot \mid x) \,\|\, \pi_i^{\mathrm{ref}}(\cdot \mid x))$.
    \EndFor
    \State \textbf{Output:} The estimated regularized NE joint policy $\hat{\pi}$.
\end{algorithmic}
\end{algorithm}

\subsection{Error Decomposition and Best-Response Analysis}
\label{subsec:error_decomp}

To analyze the suboptimality of the learned policy $\hat{\pi}$, we decouple the optimization error from the statistical error induced by the finite offline dataset by pivoting through the empirical value functions $\hat{V}$. 

\begin{restatable}[Gap Decomposition]{lemma}{reOOgapdecomposition}
\label{lem:gap_decomp_main}
For any learned joint policy $\hat{\pi}$ and any player $i \in [m]$, the unilateral exploitability gap can be exactly decomposed as:
\begin{align}
    V_i^{\dagger, \hat{\pi}_{-i}}(x) - V_{i}^{\hat{\pi}}(x) =& \underbrace{\left( \hat{V}_i^{\dagger, \hat{\pi}_{-i}}(x) - \hat{V}_{i}^{\hat{\pi}}(x) \right)}_{\text{Term I: Optimization Gap}} + \underbrace{\left( V_i^{\dagger, \hat{\pi}_{-i}}(x) - \hat{V}_i^{\dagger, \hat{\pi}_{-i}}(x) \right)}_{\text{Term II: Best-Response Error}}+ \underbrace{\left( \hat{V}_{i}^{\hat{\pi}}(x) - V_{i}^{\hat{\pi}}(x) \right)}_{\text{Term III: On-Policy Error}}.
    \label{eq:Gap decomposition}
    \hspace{-1.5em}
\end{align}
\end{restatable}
The derivation follows algebraically by adding and subtracting the estimated values $\hat{V}_i^{\dagger, \hat{\pi}_{-i}}(x)$ and $\hat{V}_i^{\hat{\pi}}(x)$ to the true unilateral exploitability gap of Player $i$. In the GANE algorithm, since $\hat{\pi}$ is an exact Regularized NE of the empirical game, Term I vanishes identically. To characterize the evaluation errors (Terms II and III), we define the pointwise regression error as 
\begin{align}
    \mathcal{Z}_i(x, \boldsymbol{a}) \coloneqq \hat{Q}_i(x, \boldsymbol{a}) - r_i^\star(x, \boldsymbol{a}).
\end{align}


\subsection{Statistical Guarantees for GANE}

Equipped with these analytical tools and the reference-anchored coverage condition (Assumption~\ref{ass:concentrability}), we establish the primary statistical guarantee for GANE.


\begin{theorem}
\label{thm:fast_rate_rone}
Under Assumption~\ref{ass:concentrability} and assuming the reward functions are realizable within $\{\mathcal{Q}_i\}_{i=1}^m$, the estimated regularized NE joint policy $\hat{\pi}$ returned by GANE satisfies, with probability at least $1-\delta$:
\begin{align}
    \mathrm{Gap}_{\mathrm{NE}}(\hat{\pi}) \le \widetilde{\mathcal{O}}\left( \frac{m \eta C_{\mathrm{uni}} \Lambda_{\eta, m} \log (|\mathcal{Q}|/\delta)}{n} \right).
\end{align}
\end{theorem}

\subsection{Proof Sketch of Theorem~\ref{thm:fast_rate_rone}}

The accelerated $\widetilde{\mathcal{O}}(1/n)$ statistical rate is made possible by a unique symmetry in evaluating product policies. We expand the mechanism below; the detailed proof is provided in Appendix~\ref{app:idealized_stat_bound}.

\paragraph{Term III evaluation.} By expanding the value definitions, the on-policy evaluation error (Term III) evaluates exactly to the expected regression error under the joint policy:
\begin{align}
    \text{Term III}_i(x) = \hat{V}_i^{\hat{\pi}}(x) - V_i^{\hat{\pi}}(x) = \mathbb{E}_{\boldsymbol{a} \sim \hat{\pi}(\cdot \mid x)} \big[ \mathcal{Z}_i(x, \boldsymbol{a}) \big].
\end{align}

\paragraph{Term II evaluation.} To bound the best-response error (Term II), we leverage the 1-smoothness of the log-partition function associated with the KL-regularizer (see Appendix~\ref{app:idealized_stat_bound}, Step 1). The difference between the true and empirical best-response values is bounded by the first-order Q-value difference plus a second-order penalty. As derived in Equation~\eqref{Delta 1 smooth Q}, this yields:
\begin{align}
    \text{Term II}_i(x) \le \mathbb{E}_{a_i \sim \hat{\pi}_i^\dagger(\cdot \mid x), \boldsymbol{a}_{-i} \sim \hat{\pi}_{-i}(\cdot \mid x)} \left[ -\mathcal{Z}_i(x, \boldsymbol{a}) \right] + \frac{\eta}{2} \left\| Q_i^{\dagger, \hat{\pi}_{-i}}(x, \cdot) - \bar{Q}_i(x, \cdot) \right\|_\infty^2,
\end{align}
where $\bar{Q}_i(x, a_i) = \mathbb{E}_{\boldsymbol{a}_{-i} \sim \hat{\pi}_{-i}}[\hat{Q}_i(x, a_i, \boldsymbol{a}_{-i})]$ is the marginalized empirical Q-function.

\paragraph{The Cancellation Mechanism.} Crucially, because $\hat{\pi}$ is the empirical equilibrium, the learned marginalized policy $\hat{\pi}_i$ is identically the empirical best response $\hat{\pi}_i^\dagger$. Furthermore, since $\hat{\pi}$ is a product policy, its evaluation distribution matches the trajectory of the best-response error. By summing the two evaluation errors, the first-order linear terms exactly cancel out:
\begin{align}
    \text{Term II}_i(x) + \text{Term III}_i(x) &\le \mathbb{E}_{\boldsymbol{a} \sim \hat{\pi}} [-\mathcal{Z}_i] + \mathbb{E}_{\boldsymbol{a} \sim \hat{\pi}} [+\mathcal{Z}_i] + \frac{\eta}{2} \left\| Q_i^{\dagger, \hat{\pi}_{-i}}(x, \cdot) - \bar{Q}_i(x, \cdot) \right\|_\infty^2 \nonumber \\
    &= \frac{\eta}{2} \left\| Q_i^{\dagger, \hat{\pi}_{-i}}(x, \cdot) - \bar{Q}_i(x, \cdot) \right\|_\infty^2.
    \label{eq:error_cancellation}
\end{align}
This cancellation removes the standard $\tilde{\mathcal{O}}(1/\sqrt{n})$ statistical bottleneck. The remaining suboptimality is bounded strictly by second-order squared error terms. These residuals are controlled by the strong convexity of the regularized objective, which, under the coverage of Assumption~\ref{ass:concentrability}, yields the fast $\widetilde{\mathcal{O}}(1/n)$ rate previously unestablished in general-sum settings.

\section{Pessimism-Free Coarse Correlated Equilibria}
\label{sec:practical_algorithm}

While GANE establishes optimal statistical limits, Nash equilibrium computation is PPAD-complete \citep{daskalakis2009complexity, chen2009settling}, making the recovery of an NE computationally intractable. This motivates targeting Coarse Correlated Equilibrium (CCE), which serves as a more tractable objective efficiently reachable via independent learning dynamics.

To provide a scalable alternative, we propose General-sum Anchored Mirror Descent (GAMD). GAMD bypasses the need for an exact equilibrium oracle. Instead, for each context evaluated, it simulates a repeated game for $T$ iterations. In each iteration $t \in [T]$, every player independently updates their policy using a KL-regularized Mirror Descent step in response to the marginalized actions of the other players. This iterative sequence converges to a Coarse Correlated Equilibrium (CCE) of the empirical game.

\begin{algorithm}[t]
\caption{General-sum Anchored Mirror Descent (GAMD)}
\label{alg:pro_md}
\begin{algorithmic}[1]
    \State \textbf{Input:} Offline dataset $\mathcal{D}$, reference policies $\pi^{\mathrm{ref}} = (\pi^{\mathrm{ref}}_1, \dots, \pi^{\mathrm{ref}}_m)$, regularization parameter $\eta > 0$, function classes $\{\mathcal{Q}_i\}_{i=1}^m$, number of iterations $T$.
    \For{player $i = 1, \dots, m$}
        \State \textbf{Empirical Q-Update:} Estimate the Q-function via least-squares regression:
        \State $\hat{Q}_i \leftarrow \argmin_{f \in \mathcal{Q}_i} \sum_{\tau=1}^n \left( f(x_{\tau}, \boldsymbol{a}_{\tau}) - r_{\tau,i} \right)^2$.
    \EndFor
    
    \State \textbf{Initialize Policies:} For all $i \in [m]$, initialize $\pi_i^{(1)}(\cdot \mid x) = \pi_i^{\mathrm{ref}}(\cdot \mid x)$ for all $x \in \mathcal{X}$.
    
    \State \textbf{Self-Play Inner Loop:}
    \For{$t = 1, \dots, T$}
        \For{player $i = 1, \dots, m$}
            \State \textbf{Compute Marginalized Q-value:} Evaluate the expected Q-value given opponents' current policies:
            \State $\bar{Q}_i^{(t)}(x, a_i) = \mathbb{E}_{\boldsymbol{a}_{-i} \sim \pi_{-i}^{(t)}(\cdot \mid x)} \big[ \hat{Q}_i(x, a_i, \boldsymbol{a}_{-i}) \big]$.
            \State \textbf{Mirror Descent Update:} Update Player $i$'s policy via the closed-form exponentiated gradient:
         \vspace{-1em}
            \begin{align}
                \pi_i^{(t+1)}(a_i \mid x) \propto \left( \pi_i^{\mathrm{ref}}(a_i \mid x) \right)^{\frac{1}{t}} \left( \pi_i^{(t)}(a_i \mid x) \right)^{\frac{t-1}{t}} \exp\left( \frac{\eta}{t} \bar{Q}_i^{(t)}(x, a_i) \right).\label{eq:md_update_closed_form}
            \end{align}
        \EndFor
    \EndFor
    
    \State \textbf{Time-Averaged Joint Policy:} Define the output policy as the uniform mixture over $T$ iterations:
    \State $\bar{\pi}(\boldsymbol{a} \mid x) = \frac{1}{T} \sum_{t=1}^T \prod_{i=1}^m \pi_i^{(t)}(a_i \mid x)$.
    
    \For{player $i = 1, \dots, m$}
        \State \textbf{Value Update:} Update the empirical value function using the time-averaged policy:
        \State $\hat{V}_i(x) \leftarrow \mathbb{E}_{\boldsymbol{a} \sim \bar{\pi}(\cdot \mid x)} \big[ \hat{Q}_i(x, \boldsymbol{a}) \big] - \eta^{-1} \mathrm{KL}\big(\bar{\pi}_i(\cdot \mid x) \,\|\, \pi^{\mathrm{ref}}_i(\cdot \mid x)\big)$.
    \EndFor
    \State \textbf{Output:} The learned time-averaged joint policy $\bar{\pi}$.
\end{algorithmic}
\end{algorithm}

\subsection{Theoretical Guarantees of GAMD}


To formalize the convergence of GAMD, we first bound the optimization error of the inner loop. Because the KL-divergence penalty transforms the update into a strongly concave objective, applying Mirror Descent yields an accelerated optimization rate. We specialize the standard convergence results of Online Convex Optimization \citep{shalev2025online} to our setting of regularized payoff sequences, where the optimization is anchored to the fixed reference policy $\pi^{\mathrm{ref}}$.     We formalize it in the following lemma (proof provided in Appendix~\ref{app:lemma omd proof}):

\begin{lemma}[OMD External Regret]
\label{lem:omd_external_regret}
Let $f^{(t)}(\pi) = \langle \pi, Q^{(t)} \rangle - \eta^{-1} \mathrm{KL}(\pi \,\|\, \pi^{\mathrm{ref}})$ be a sequence of objectives where $\|Q^{(t)}\|_\infty \le 1$. Applying KL-regularized Online Mirror Descent with the decaying stepsize schedule $\gamma_t = 1/t$ as in Algorithm~\ref{alg:pro_md} yields an average external regret bounded by:
\begin{align}
    \max_{\pi \in \Delta(\mathcal{A})} \frac{1}{T} \sum_{t=1}^T f^{(t)}(\pi) - \frac{1}{T} \sum_{t=1}^T f^{(t)}(\pi^{(t)}) \le \mathcal{O}\left( \frac{\eta \log T}{T} \right).
\end{align}
\end{lemma}

By combining this $\widetilde{\mathcal{O}}(1/T)$ optimization rate with the statistical error induced by finite offline samples, we establish the final suboptimality guarantee for GAMD. 


\begin{theorem}
\label{thm:gamd_cce_rate}
Under Assumption~\ref{ass:concentrability},  and assuming the reward functions are realizable within $\{\mathcal{Q}_i\}_{i=1}^m$,  the policy $\bar{\pi}$ returned by GAMD satisfies, with probability at least $1-\delta$:
\begin{align}
    \mathrm{Gap}_{\mathrm{CCE}}(\bar{\pi}) \le \underbrace{\widetilde{\mathcal{O}}\left( \frac{m \eta \log T}{T} \right)}_{\text{Optimization Error}} + \underbrace{\widetilde{\mathcal{O}}\left( m \sqrt{ \frac{\Lambda_{\eta, m} C_{\mathrm{uni}} \log (|\mathcal{Q}|/\delta)}{n} } \right)}_{\text{Statistical Correlation Residual}} + \underbrace{\widetilde{\mathcal{O}}\left( \frac{1}{n} \right)}_{\text{Higher-Order Noise}}
\end{align}
Consequently, for $T \ge \sqrt{n}$, GAMD recovers a coarse correlated equilibrium at the standard $\widetilde{\mathcal{O}}(1/\sqrt{n})$ statistical rate.
\end{theorem}

\subsection{Proof Sketch of Theorem~\ref{thm:gamd_cce_rate}}

Using the decomposition from Lemma~\ref{lem:gap_decomp_main}, Term I corresponds to the average external regret of the mirror descent updates across $T$ iterations. By standard Online Convex Optimization, this optimization error decays at $\widetilde{\mathcal{O}}(m \eta / T)$ (Lemma~\ref{lem:omd_external_regret}).

For the statistical errors (Terms II and III), GAMD faces a fundamental structural barrier. The recovered time-averaged policy $\bar{\pi}$ is a \textit{correlated} joint mixture, meaning the independent product distribution differs from the actual joint policy: $\bar{\pi}_i^\dagger \times \bar{\pi}_{-i} \neq \bar{\pi}$. Consequently, the linear regression errors $-\mathcal{Z}_i$ (from the best response) and $+\mathcal{Z}_i$ (from on-policy evaluation) are integrated over mismatched distributions and do not algebraically cancel:
\begin{align}
    \text{Term II}_i(x) + \text{Term III}_i(x) \approx \mathbb{E}_{\boldsymbol{a} \sim (\bar{\pi}_i^\dagger, \bar{\pi}_{-i})} [ -\mathcal{Z}_i(x, \boldsymbol{a}) ] + \mathbb{E}_{\boldsymbol{a} \sim \bar{\pi}} [ +\mathcal{Z}_i(x, \boldsymbol{a}) ] \neq 0.
\end{align}

Without cancellation, we bound the absolute magnitude of these linear expectations using Cauchy-Schwarz, $\mathbb{E}[|\mathcal{Z}_i|] \le \sqrt{\mathbb{E}[\mathcal{Z}_i^2]}$. This relates the linear terms back to the $\widetilde{\mathcal{O}}(1/n)$ in-sample squared regression error, yielding an $\tilde{\mathcal{O}}(1/\sqrt{n})$ bound. Thus, the inherent correlation of the CCE forces the statistical evaluation error to dominate at the standard minimax rate of $\widetilde{\mathcal{O}}(1/\sqrt{n})$. Setting $T \ge \sqrt{n}$ balances the optimization and statistical errors. The complete proof is in Appendix~\ref{app:proof_pro_md_cce}.

\section{Related Work}
\label{sec:related_work}

\paragraph{Pessimism in Offline Reinforcement Learning.}
Offline reinforcement learning (RL) focuses on policy optimization from pre-collected datasets without the benefit of interactive exploration. The primary theoretical obstacle in this regime is distribution shift, wherein the target policy's visitation frequency deviates from the logged behavior, inducing severe overestimation and extrapolation error \citep{levine2020offline, liu2024doubly, liu2024efficient, chen2025efficient, liu2025efficient}. To counteract this bias, researchers utilize the principle of \textit{pessimism}, which provides a robust framework for establishing sample-efficient guarantees by penalizing values in poorly explored state-action regions \citep{liu2020provably, rashidinejad2021bridging, jin2021pessimism, xie2021bellman, uehara2021pessimistic, zhan2022offline}. These conservative strategies, often implemented via lower confidence bounds or conservative value iterations, have been proven minimax optimal under various single-policy concentrability settings \citep{li2024settling}. 

In game-theoretic environments, the complexity of this shift is intensified by strategic interactions, requiring sufficient data coverage of unilateral deviations for each agent to ensure successful equilibrium recovery \citep{cui2022offline}. Consequently, the current literature on offline Markov games relies heavily on these pessimistic foundations to mitigate information gaps within the static logs \citep{cui2022offline, zhong2022pessimistic, zhang2023offline}. In contrast, our approach establishes that KL regularization suffices to stabilize learning without relying on such explicit pessimistic mechanisms.

\paragraph{Equilibrium Recovery in Multi-Agent Games.}
The transition from single-agent RL to multi-agent environments necessitates a shift from value optimality to equilibrium stability, targeting solution concepts such as Nash Equilibrium and Coarse Correlated Equilibrium (CCE) \citep{roughgarden2016twenty, cui2022provably, cui2022offline}. To recover these equilibria from fixed datasets, existing methodologies largely adapt the principle of pessimism to the game-theoretic setting, identifying minimax-optimal strategies through conservative value iteration and strategy-wise uncertainty bonuses \citep{zhong2022pessimistic, cui2022provably, zhang2023offline, yan2024model}. 

The theoretical validity of these pessimistic approaches is grounded in the notion of \textit{unilateral concentrability}—the requirement that the offline logs sufficiently cover any individual agent's deviations from a target equilibrium \citep{cui2022provably}. However, anchoring this coverage requirement to an unknown optimal equilibrium $\pi^*$ introduces a fundamental conceptual circularity, as the very policy required to define data sufficiency is the target of the learning process itself. Our work addresses this by anchoring the learning objective to a fixed and known reference policy $\pi^{\mathrm{ref}}$ \citep{chen2026fast}, transforming the coverage requirement into a more interpretable condition that remains grounded in the support of the offline data.

\paragraph{KL Regularization and Behavioral Anchoring.}
KL regularization relative to a fixed reference policy has emerged as a fundamental tool for enforcing behavioral constraints and incorporating prior knowledge into reinforcement learning frameworks~\citep{xiong2023iterative, munos2024nash}. By penalizing deviations from an anchor distribution, such as a pre-trained model or human demonstrations, this mechanism effectively mitigates policy drift and stabilizes the resulting learning dynamics~\citep{ye2024online, nayak2025achieving}. The most prominent application of this principle is found in the alignment of large language models through Reinforcement Learning from Human Feedback (RLHF), where a KL divergence penalty ensures the optimized model remains within the trusted region of the initial distribution~\citep{ouyang2022training, rafailov2023direct}. 

While regularized objectives have been extensively analyzed across both single-agent and multi-agent regimes~\citep{xie2024exploratory, zhao2025logarithmic, zhang2025improving, zhang2025iterative}, current theoretical results typically suggest that regularization primarily serves as a numerical stabilizer, yielding sample complexity rates that mirror those of their unregularized counterparts~\citep{ye2024online}. Consequently, the capacity of KL regularization to function as a standalone, implicit pessimism mechanism  in offline general-sum games has remained unestablished. Our work bridges this gap by demonstrating that regularized KL anchoring suffices for stable equilibrium recovery without explicit bonuses.

\paragraph{Statistical Efficiency and Fast Rates.}
Most offline learning guarantees characterize the standard $\widetilde{\mathcal{O}}(1/\sqrt{n})$ statistical rate~\citep{jin2021pessimism, shi2022pessimistic}. While recent investigations demonstrated that an accelerated $\widetilde{\mathcal{O}}(1/n)$ rate is achievable in two-player zero-sum games \citep{zhang2026beyond, chen2026offline}, those results rely on the minimax structure and skew-symmetry inherent to strictly competitive dynamics. In contrast, general-sum games are fundamentally more complex; whereas finding an equilibrium is already PPAD-complete in the two-player general-sum case~\citep{chen2009settling}, the multi-player setting remains equally intractable while introducing more complex strategic interactions~\citep{papadimitriou1994complexity, daskalakis2009complexity}. 

We broaden this theoretical landscape in two directions. First, we establish that the fast $\widetilde{\mathcal{O}}(1/n)$ rate persists for Nash equilibria in general-sum games, proving that independence—rather than zero-sum symmetry—is the key driver of this acceleration. Second, for the computationally tractable CCE, we provide a pessimism-free framework that achieves the standard $\widetilde{\mathcal{O}}(1/\sqrt{n})$ rate. Together, these findings establish a tractable foundation for pessimism-free learning in multi-player general-sum games.

\section{Discussion}
\label{sec:discussion}

While our work establishes KL regularization as a standalone, pessimism-free mechanism for equilibrium learning in general-sum games, several promising directions for future research remain. First, extending this framework to multi-step Markov games is essential to address the accumulation of estimation errors across temporal horizons. Second, adapting our regularized dynamics to partially observable environments or extensive-form games would broaden applicability to real-world systems involving hidden states and asymmetric information. Finally, enhancing robustness against misspecified reference policies—potentially via adaptive regularization schedules—is a key priority for maintaining reliability in highly non-stationary or adversarial strategic settings.

\section{Conclusion}
\label{sec:conclusion}
In this work, we have demonstrated that KL regularization serves as a robust and standalone mechanism for stabilizing offline multi-agent learning, effectively mitigating distribution shift without the need for explicit pessimistic penalties. By shifting the coverage requirement from unknown optimal policies to a fixed and known reference anchor, we established a theoretical framework that bypasses the conceptual circularity inherent in prior literature. We proposed GANE, which achieves an accelerated $\widetilde{\mathcal{O}}(1/n)$ statistical rate for Nash equilibria by leveraging the product structure of independent policies. Furthermore, we provided GAMD, a decentralized iterative framework that recovers Coarse Correlated Equilibria at the standard $\widetilde{\mathcal{O}}(1/\sqrt{n})$ rate without explicit pessimistic bonuses. Together, these results establish KL regularization as a tractable, pessimism-free alternative for establishing optimal statistical rates in offline game-theoretic learning.

\section*{Acknowledgment}
Yuheng Zhang is supported by a fellowship from the Amazon-Illinois Center on AI for Interactive Conversational Experiences (AICE). Nan Jiang acknowledges funding support from NSF CNS-2112471, NSF CAREER IIS-2141781, and Sloan Fellowship.

%

\bibliographystyle{apalike}
\bibliography{bibliography}

\clearpage
\appendix

\section{Proof}
\subsection{Proof of Theorem~\ref{thm:fast_rate_rone}}
\label{app:idealized_stat_bound}

\subsubsection{Decomposition of the Total Exploitability Gap}
\label{app:gap_decomposition}

To bound the suboptimality of the learned joint policy $\hat{\pi}$, we decompose the Nash Equilibrium gap into an empirical optimization error and statistical evaluation errors by pivoting through the estimated value functions $\hat{V}$ and $\hat{Q}$.

Recall the definition of the NE gap (Total Exploitability) for the learned policy $\hat{\pi}$:
\begin{equation}
    \mathrm{Gap}_{\mathrm{NE}}(\hat{\pi}) = \sum_{i=1}^m \mathbb{E}_{x \sim \rho} \left[ V_i^{\dagger, \hat{\pi}_{-i}}(x) - V_i^{\hat{\pi}}(x) \right].
\end{equation}

We define the estimated best-response value for Player $i$ under the empirical game evaluated by $\hat{Q}$ as:
\begin{equation}
    \hat{V}_i^{\dagger, \hat{\pi}_{-i}}(x) := \max_{\pi'_i} \hat{V}_i^{\pi'_i, \hat{\pi}_{-i}}(x).
\end{equation}

Here we restate Lemma~\ref{lem:gap_decomp_main}:
\reOOgapdecomposition*


\paragraph{Bounding the Decomposition Terms:}
This decomposition isolates the distinct sources of error, allowing us to bound the gap systematically:
\begin{itemize}
    \item \textbf{Term I} represents the empirical external regret of Player $i$ operating strictly on the estimated game $\hat{Q}$. Because the joint policy $\hat{\pi}$ is obtained via an oracle exactly solving the empirical games, this term strictly corresponds to the optimization error and vanishes identically.
    \item \textbf{Term II} and \textbf{Term III} represent the statistical evaluation errors. By applying algebraic expansions, these terms can be unrolled into the sum of regression errors over the contextual distributions, respectively. These errors are precisely controlled by our Unilateral Concentrability assumption and converted into squared errors via the KL-divergence strong convexity.
\end{itemize}

\subsubsection{Bounding the Statistical Error for the Idealized Algorithm}
Let $\hat{\pi}$ be the joint policy returned by the idealized GANE algorithm (Algorithm \ref{alg:ideal_GANE}). Because $\hat{\pi}$ is an exact regularized Nash Equilibrium of the estimated game $\hat{Q}$, no player can unilaterally improve their empirical value. Thus, for all $i \in [m]$, the empirical optimization gap vanishes:
\begin{align}
    \text{Term I}_i = \hat{V}_i^{\dagger, \hat{\pi}_{-i}}(x) - \hat{V}_i^{\hat{\pi}}(x) = 0.
\end{align}
Therefore, the total unilateral exploitability gap is bounded strictly by the statistical evaluation errors:
\begin{align}
    \mathrm{Gap}_{\mathrm{NE}}(\hat{\pi}) \le \sum_{i=1}^m \left( \text{Term II}_i + \text{Term III}_i \right),
\end{align}
where $\text{Term II}_i = V_i^{\dagger, \hat{\pi}_{-i}}(x) - \hat{V}_i^{\dagger, \hat{\pi}_{-i}}(x)$ and $\text{Term III}_i = \hat{V}_i^{\hat{\pi}}(x) - V_i^{\hat{\pi}}(x)$.

Since $\hat{\pi}$ is an exact NE on the estimated values, the marginalized policy $\hat{\pi}_i$ is the empirical best response to $\hat{\pi}_{-i}$, implying $\hat{V}_i^{\dagger, \hat{\pi}_{-i}} = \hat{V}_i^{\hat{\pi}}$. We analyze the sum of these evaluation errors by first bounding the local value difference of the best response.

\paragraph{Step 1: 1-Smoothness of the Log-Partition Function.}
Fix a context $x$. We define the value difference of the best response as $\Delta_i(x) \coloneqq V_i^{\dagger, \hat{\pi}_{-i}}(x) - \hat{V}_i^{\dagger, \hat{\pi}_{-i}}(x)$.
For notational clarity, we define the marginalized empirical Q-function as $\bar{Q}_i(x, a_i) \coloneqq \mathbb{E}_{\boldsymbol{a}_{-i} \sim \hat{\pi}_{-i}}[\hat{Q}_i(x, a_i, \boldsymbol{a}_{-i})]$ and the true best-response Q-value as $Q_i^{\dagger, \hat{\pi}_{-i}}(x, a_i) \coloneqq \mathbb{E}_{\boldsymbol{a}_{-i} \sim \hat{\pi}_{-i}}[r_i^\star(x, \boldsymbol{a})]$.
Both the true best response and the empirical best response values can be expressed exactly via the log-partition function $\Phi(\theta) = \log \sum_a \exp(\theta_a)$. 
We define the natural parameters (logits) for the true and empirical best responses respectively:
\begin{align}
    \theta_i^{\dagger}(x,\cdot) &\coloneqq \eta Q_i^{\dagger, \hat{\pi}_{-i}}(x, \cdot) + \log \pi_{\mathrm{ref}, i}(\cdot \mid x), \\
    \hat{\theta}_i(x,\cdot) &\coloneqq \eta \bar{Q}_i(x, \cdot) + \log \pi_{\mathrm{ref}, i}(\cdot \mid x).
\end{align}
The values are exactly $V_i^{\dagger}(x) = \eta^{-1} \Phi(\theta_i^{\dagger})$ and $\hat{V}_i^{\dagger}(x) = \eta^{-1} \Phi(\hat{\theta}_i)$.
Because the log-partition function $\Phi$ is 1-smooth with respect to the $\ell_\infty$-norm, its Bregman divergence (which corresponds to the KL divergence) is strictly bounded by the squared $\ell_\infty$ distance:
\begin{align}
    \Phi(\theta^\dagger) - \Phi(\hat{\theta}) \le \langle \nabla \Phi(\hat{\theta}), \theta^\dagger - \hat{\theta} \rangle + \frac{1}{2} \|\theta^\dagger - \hat{\theta}\|_\infty^2.
\end{align}
Crucially, the gradient map $\nabla \Phi(\hat{\theta})$ yields exactly the probabilities of the empirical best response, which is the learned policy $\hat{\pi}_i$. Dividing by $\eta$, we obtain the local value difference bound:
\begin{align}
    \Delta_i(x) \le \mathbb{E}_{a_i \sim \hat{\pi}_i(\cdot|x)} \left[ Q_i^{\dagger, \hat{\pi}_{-i}}(x, a_i) - \bar{Q}_i(x, a_i) \right] + \frac{\eta}{2} \left\| Q_i^{\dagger, \hat{\pi}_{-i}}(x, \cdot) - \bar{Q}_i(x, \cdot) \right\|_\infty^2.
    \label{Delta 1 smooth Q}
\end{align}
\begin{proof}
By definition, the local value difference is given by the scaled difference of the log-partition functions:
\begin{align}
    \Delta_i(x) 
    &= V_i^{\dagger, \hat{\pi}_{-i}}(x) - \hat{V}_i^{\dagger, \hat{\pi}_{-i}}(x) \nonumber \\
    &= \eta^{-1} \Phi(\theta_i^{\dagger}(x, \cdot)) - \eta^{-1} \Phi(\hat{\theta}_i(x, \cdot)) \nonumber \\
    &= \eta^{-1} \Big( \Phi(\theta_i^{\dagger}) - \Phi(\hat{\theta}_i) \Big).
\end{align}

Applying the 1-smoothness property of $\Phi$ to the term in the parentheses:
\begin{align}
    \Delta_i(x) 
    &\le \eta^{-1} \left[ \langle \nabla \Phi(\hat{\theta}_i), \theta_i^{\dagger} - \hat{\theta}_i \rangle + \frac{1}{2} \|\theta_i^{\dagger} - \hat{\theta}_i\|_\infty^2 \right].
    \label{eq:smoothness_subbed}
\end{align}

We now evaluate the components of this upper bound. First, the gradient of the log-partition function $\nabla \Phi(\theta)$ yields the Softmax probabilities. Since $\hat{\theta}_i$ represents the logits of the empirical game, its Softmax is exactly the learned empirical policy:
\begin{align}
    \nabla \Phi(\hat{\theta}_i)_{a_i} = \hat{\pi}_i(a_i|x).
\end{align}

Second, we evaluate the difference in logits. By definition, $\theta_i^{\dagger}(x,a_i) = \eta Q_i^{\dagger, \hat{\pi}_{-i}}(x, a_i) + \log \pi_{\mathrm{ref}, i}(a_i \mid x)$, and similarly for $\hat{\theta}_i$. When we subtract them, the reference policy perfectly cancels out:
\begin{align}
    \theta_i^{\dagger}(x, a_i) - \hat{\theta}_i(x, a_i) 
    &= \Big( \eta Q_i^{\dagger, \hat{\pi}_{-i}}(x, a_i) + \log \pi_{\mathrm{ref}, i}(a_i|x) \Big) - \Big( \eta \bar{Q}_i(x, a_i) + \log \pi_{\mathrm{ref}, i}(a_i|x) \Big) \nonumber \\
    &= \eta \Big( Q_i^{\dagger, \hat{\pi}_{-i}}(x, a_i) - \bar{Q}_i(x, a_i) \Big).
\end{align}

Substituting these evaluated components back into the inner product term of Equation \eqref{eq:smoothness_subbed}:
\begin{align}
    \langle \nabla \Phi(\hat{\theta}_i), \theta_i^{\dagger} - \hat{\theta}_i \rangle 
    &= \sum_{a_i \in \mathcal{A}_i} \hat{\pi}_i(a_i|x) \cdot \eta \Big( Q_i^{\dagger, \hat{\pi}_{-i}}(x, a_i) - \bar{Q}_i(x, a_i) \Big) \nonumber \\
    &= \eta \mathbb{E}_{a_i \sim \hat{\pi}_i(\cdot|x)} \left[ Q_i^{\dagger, \hat{\pi}_{-i}}(x, a_i) - \bar{Q}_i(x, a_i) \right].
\end{align}

Next, we substitute the logit difference into the squared $\ell_\infty$ norm term:
\begin{align}
    \frac{1}{2} \|\theta_i^{\dagger} - \hat{\theta}_i\|_\infty^2 
    &= \frac{1}{2} \left\| \eta \Big( Q_i^{\dagger, \hat{\pi}_{-i}}(x, \cdot) - \bar{Q}_i(x, \cdot) \Big) \right\|_\infty^2 \nonumber \\
    &= \frac{\eta^2}{2} \left\| Q_i^{\dagger, \hat{\pi}_{-i}}(x, \cdot) - \bar{Q}_i(x, \cdot) \right\|_\infty^2.
\end{align}

Finally, plugging these two terms back into Equation \eqref{eq:smoothness_subbed} and multiplying through by $\eta^{-1}$:
\begin{align}
    \Delta_i(x) 
    &\le \eta^{-1} \Bigg[ \eta \mathbb{E}_{a_i \sim \hat{\pi}_i(\cdot|x)} \left[ Q_i^{\dagger, \hat{\pi}_{-i}}(x, a_i) - \bar{Q}_i(x, a_i) \right] + \frac{\eta^2}{2} \left\| Q_i^{\dagger, \hat{\pi}_{-i}}(x, \cdot) - \bar{Q}_i(x, \cdot) \right\|_\infty^2 \Bigg] \nonumber \\
    &= \mathbb{E}_{a_i \sim \hat{\pi}_i(\cdot|x)} \left[ Q_i^{\dagger, \hat{\pi}_{-i}}(x, a_i) - \bar{Q}_i(x, a_i) \right] + \frac{\eta}{2} \left\| Q_i^{\dagger, \hat{\pi}_{-i}}(x, \cdot) - \bar{Q}_i(x, \cdot) \right\|_\infty^2.
\end{align}
\end{proof}

\paragraph{Step 2: Evaluating Term II via the Pointwise Error Identity.}
We introduce the regression error for Player $i$:
\begin{align}
    \mathcal{Z}_i(x, \boldsymbol{a}) \coloneqq \hat{Q}_i(x, \boldsymbol{a}) - r_i^\star(x, \boldsymbol{a}).
\end{align}
The Q-value difference expands algebraically as:
\begin{align}
    Q_i^{\dagger, \hat{\pi}_{-i}}(x, a_i) - \bar{Q}_i(x, a_i) = \mathbb{E}_{\boldsymbol{a}_{-i} \sim \hat{\pi}_{-i}} \left[ -\mathcal{Z}_i(x, \boldsymbol{a}) \right].
\label{Q difference in Z}
\end{align}
To see \eqref{Q difference in Z}, recall the definitions for the true best-response Q-value, the estimated Q-value, and the regression error for Player $i$:
\begin{align}
    Q_i^{\dagger, \hat{\pi}_{-i}}(x, a_i) &= \mathbb{E}_{\boldsymbol{a}_{-i} \sim \hat{\pi}_{-i}} \left[ r_i^\star(x, \boldsymbol{a}) \right], \\
    \bar{Q}_i(x, a_i) &= \mathbb{E}_{\boldsymbol{a}_{-i} \sim \hat{\pi}_{-i}} \big[ \hat{Q}_i(x, \boldsymbol{a}) \big], \\
    \mathcal{Z}_i(x, \boldsymbol{a}) &\coloneqq \hat{Q}_i(x, \boldsymbol{a}) - r_i^\star(x, \boldsymbol{a}).
\end{align}

By expanding the Q-value difference and applying direct substitution, we obtain:
\begin{align}
    &Q_i^{\dagger, \hat{\pi}_{-i}}(x, a_i) - \bar{Q}_i(x, a_i) \nonumber \\
    &= \mathbb{E}_{\boldsymbol{a}_{-i} \sim \hat{\pi}_{-i}} \left[ r_i^\star(x, \boldsymbol{a}) \right] - \mathbb{E}_{\boldsymbol{a}_{-i} \sim \hat{\pi}_{-i}} \big[ \hat{Q}_i(x, \boldsymbol{a}) \big] \nonumber \\
    &= \mathbb{E}_{\boldsymbol{a}_{-i} \sim \hat{\pi}_{-i}} \left[ r_i^\star(x, \boldsymbol{a}) - \hat{Q}_i(x, \boldsymbol{a}) \right] \nonumber \\
    \intertext{Grouping the terms strictly isolates the exact definition of $\mathcal{Z}_i$:}
    &= \mathbb{E}_{\boldsymbol{a}_{-i} \sim \hat{\pi}_{-i}} \left[ -\big(\hat{Q}_i(x, \boldsymbol{a}) - r_i^\star(x, \boldsymbol{a})\big) \right] \nonumber \\
    &= \mathbb{E}_{\boldsymbol{a}_{-i} \sim \hat{\pi}_{-i}} \left[ -\mathcal{Z}_i(x, \boldsymbol{a}) \right].
    \label{Q difference in Delta recursive}
\end{align}

Recall that the Total Exploitability of the game is evaluated in expectation over the context distribution $x \sim \rho$. Therefore, to bound the total gap, we evaluate the expected value of our decomposition terms. 

For Player $i$, we define the pointwise value difference of the best response as $\Delta_i(x) \coloneqq V_i^{\dagger, \hat{\pi}_{-i}}(x) - \hat{V}_i^{\dagger, \hat{\pi}_{-i}}(x)$. 
By direct algebraic substitution, the linear Q-value difference is strictly equal to the negative regression error:
\begin{align}
    Q_i^{\dagger, \hat{\pi}_{-i}}(x, a_i) - \bar{Q}_i(x, a_i) = \mathbb{E}_{\boldsymbol{a}_{-i} \sim \hat{\pi}_{-i}} \left[ -\mathcal{Z}_i(x, \boldsymbol{a}) \right].
\end{align}
Taking the expectation over the context distribution $x \sim \rho$, Term II is bounded by:
\begin{align}
    \mathbb{E}_{x \sim \rho}[\text{Term II}_i(x)] \le \mathbb{E}_{x \sim \rho, \boldsymbol{a} \sim (\hat{\pi}_i^\dagger, \hat{\pi}_{-i})} \left[ -\mathcal{Z}_i(x, \boldsymbol{a}) \right] + \frac{\eta}{2} \mathbb{E}_{x \sim \rho} \left[ \left\| Q_i^{\dagger, \hat{\pi}_{-i}}(x, \cdot) - \bar{Q}_i(x, \cdot) \right\|_\infty^2 \right].
\end{align}

\paragraph{Step 3: Evaluating Term III via Pointwise Error Identity.}
Next, we apply the similar algebraic derivation on Term III.
\begin{align}
    \mathbb{E}_{x \sim \rho}[\text{Term III}_i(x)] &= \mathbb{E}_{x \sim \rho} \left[ \hat{V}_i^{\hat{\pi}}(x) - V_i^{\hat{\pi}}(x) \right] \nonumber \\
    \intertext{We define the local joint policy evaluation error as $\delta_i(x) \coloneqq \hat{V}_i^{\hat{\pi}}(x) - V_i^{\hat{\pi}}(x)$. Expanding both value functions according to their definitions, the KL regularization terms exactly cancel out because both are evaluated under the identical policy $\hat{\pi}$:}
    \delta_i(x) &= \left( \mathbb{E}_{\boldsymbol{a} \sim \hat{\pi}(\cdot|x)} \big[ \hat{Q}_i(x, \boldsymbol{a}) \big] - \text{Reg}_i(\hat{\pi}) \right) \nonumber \\
    &\quad - \left( \mathbb{E}_{\boldsymbol{a} \sim \hat{\pi}(\cdot|x)} \Big[ r_i^\star(x, \boldsymbol{a}) \Big] - \text{Reg}_i(\hat{\pi}) \right) \nonumber \\
    &= \mathbb{E}_{\boldsymbol{a} \sim \hat{\pi}(\cdot|x)} \left[ \hat{Q}_i(x, \boldsymbol{a}) - r_i^\star(x, \boldsymbol{a}) \right] \nonumber \\
    \intertext{Grouping the terms recovers the exact definition of the regression error $\mathcal{Z}_i$:}
    &= \mathbb{E}_{\boldsymbol{a} \sim \hat{\pi}(\cdot|x)} \Bigg[ \underbrace{\hat{Q}_i(x, \boldsymbol{a}) - r_i^\star(x, \boldsymbol{a})}_{= \mathcal{Z}_i(x, \boldsymbol{a})} \Bigg] \nonumber \\
    &= \mathbb{E}_{\boldsymbol{a} \sim \hat{\pi}(\cdot|x)} \left[ \mathcal{Z}_i(x, \boldsymbol{a}) \right] \nonumber \\
    \intertext{We evaluate this identity in expectation over the context distribution $x \sim \rho$:}
    \mathbb{E}_{x \sim \rho} \left[ \delta_i(x) \right] &= \mathbb{E}_{x \sim \rho} \Bigg[ \mathbb{E}_{\boldsymbol{a} \sim \hat{\pi}(\cdot|x)} \Big[ \mathcal{Z}_i(x, \boldsymbol{a}) \Big] \Bigg] \nonumber \\
    &= \mathbb{E}_{x \sim \rho,\, \boldsymbol{a} \sim \hat{\pi}(\cdot|x)} \left[ \mathcal{Z}_i(x, \boldsymbol{a}) \right].
\end{align}
Thus, taking the expectation over the context distribution, we obtain:
\begin{align}
    \mathbb{E}_{x \sim \rho}[\text{Term III}_i(x)] = \mathbb{E}_{x \sim \rho}\left[ \hat{V}_i^{\hat{\pi}}(x) - V_i^{\hat{\pi}}(x) \right] = \mathbb{E}_{x \sim \rho,\, \boldsymbol{a} \sim \hat{\pi}(\cdot|x)} \left[ \mathcal{Z}_i(x, \boldsymbol{a}) \right].
\end{align}

When adding the expected Term II and Term III, the linear evaluation errors $-\mathcal{Z}_i$ and $+\mathcal{Z}_i$ completely cancel out. This leaves the total expected gap bounded strictly by the squared errors:
\begin{align}
    \mathbb{E}_{x \sim \rho} \big[ \text{Term II}_i(x) + \text{Term III}_i(x) \big] \le \frac{\eta}{2} \mathbb{E}_{x \sim \rho} \left[ \left\| Q_i^{\dagger, \hat{\pi}_{-i}}(x, \cdot) - \bar{Q}_i(x, \cdot) \right\|_\infty^2 \right].
\label{Term 23 in Q difference}
\end{align}

\paragraph{Step 4: Unrolling the Linear Error and Cauchy-Schwarz.}
Taking the absolute value of the Q-value difference identity \eqref{Q difference in Delta recursive} and applying the triangle inequality yields the pointwise linear bound:
\begin{align}
    \left\| Q_i^{\dagger, \hat{\pi}_{-i}}(x, \cdot) - \bar{Q}_i(x, \cdot) \right\|_\infty \le \max_{a_i} \mathbb{E}_{\boldsymbol{a}_{-i} \sim \hat{\pi}_{-i}} \Big[ |\mathcal{Z}_i(x, \boldsymbol{a})| \Big].
\end{align}

Notice that taking the maximum over the action $a_i$ is equivalent to taking the supremum over all possible state-conditioned policies for Player $i$. Because the opponent's policy remains fixed to $\hat{\pi}_{-i}$, this shifts the action evaluation distribution from the learned joint policy to a unilateral deviation distribution $(\pi_i', \hat{\pi}_{-i})$. Thus, we can bound the point-wise error by taking the supremum over all such unilateral deviations:
\begin{align}
    \left\| Q_i^{\dagger, \hat{\pi}_{-i}}(x, \cdot) - \bar{Q}_i(x, \cdot) \right\|_\infty \le \sup_{\pi_i'} \mathbb{E}_{\boldsymbol{a} \sim (\pi_i', \hat{\pi}_{-i})} \big[ |\mathcal{Z}_i(x, \boldsymbol{a})| \big],
\end{align}
where the inner action $a_i$ is selected by the deviation policy $\pi_i'$. 

We now square both sides. Applying Jensen's inequality for the inner expectation, we obtain:
\begin{align}
    \left\| Q_i^{\dagger, \hat{\pi}_{-i}}(x, \cdot) - \bar{Q}_i(x, \cdot) \right\|_\infty^2 \le \sup_{\pi_i'} \mathbb{E}_{\boldsymbol{a} \sim (\pi_i', \hat{\pi}_{-i})} \big[ \mathcal{Z}_i(x, \boldsymbol{a})^2 \big] = \max_{a_i} \mathbb{E}_{\boldsymbol{a}_{-i} \sim \hat{\pi}_{-i}} \big[ \mathcal{Z}_i(x, \boldsymbol{a})^2 \big].
\end{align}
Substituting this squared unrolled error back into the sum of Term II and Term III bounds gives:
\begin{align} 
     \mathbb{E}_{x \sim \rho} \big[ \text{Term II}_i(x) + \text{Term III}_i(x) \big] &\le \frac{\eta}{2} \mathbb{E}_{x \sim \rho} \left[ \left\| Q_i^{\dagger, \hat{\pi}_{-i}}(x, \cdot) - \bar{Q}_i(x, \cdot) \right\|_\infty^2 \right] \nonumber \\
    &\le \frac{\eta}{2} \mathbb{E}_{x \sim \rho} \left[ \max_{a_i} \mathbb{E}_{\boldsymbol{a}_{-i} \sim \hat{\pi}_{-i}} \big[ \mathcal{Z}_i(x, \boldsymbol{a})^2 \big] \right].
    \label{Term 2 3 in Z unshift}
\end{align}


\begin{lemma}[Uniform Value Bound]
\label{lem:value_bound}
For any player $i \in [m]$ and any fixed opponent joint policy $\nu_{-i}$, the regularized best-response value function $V_i^{\dagger, \nu_{-i}}(x) = \max_{\pi_i} V_i^{\pi_i, \nu_{-i}}(x)$ satisfies:
\begin{align}
    0 \le V_i^{\dagger, \nu_{-i}}(x) \le 1. \label{eq:V-infty-bound-lse}
\end{align}
Consequently, the uniform infinity norm is strictly bounded: $\left\| V_i^{\dagger, \nu_{-i}} \right\|_\infty \le 1$.
\end{lemma}

\begin{proof}
For a fixed context $x$, the best-response value for Player $i$ against $\nu_{-i}$ can be written in its exact log-sum-exp form:
\begin{align}
    V_i^{\dagger, \nu_{-i}}(x) = \eta^{-1} \log \sum_{a_i \in \mathcal{A}_i} \pi_{\mathrm{ref}, i}(a_i|x) \exp\left( \eta Q_i^{\dagger, \nu_{-i}}(x, a_i) \right),
\end{align}
where $Q_i^{\dagger, \nu_{-i}}(x, a_i) = \mathbb{E}_{\boldsymbol{a}_{-i} \sim \nu_{-i}}[ r_i^\star(x, \boldsymbol{a}) ]$. 
Using the bounds of the deterministic reward $r_i^\star \in [0, 1]$, the Q-value is uniformly bounded by $Q_i^{\dagger, \nu_{-i}}(x, a_i) \in [0, 1]$.

Because the reference policy $\pi_{\mathrm{ref}, i}(\cdot|x)$ is a valid probability distribution that sums to $1$, we apply the standard bounds of the log-sum-exp function: $\eta^{-1}\log\sum p_j \exp(\eta q_j) \in [\min q_j, \max q_j]$. 
Therefore, the best-response value is strictly bounded by the minimum and maximum expected Q-values:
\begin{align}
    \min_{a_i} Q_i^{\dagger, \nu_{-i}}(x, a_i) \le V_i^{\dagger, \nu_{-i}}(x) \le \max_{a_i} Q_i^{\dagger, \nu_{-i}}(x, a_i).
\end{align}
Since all Q-values are bounded in $[0, 1]$, we conclude that $V_i^{\dagger, \nu_{-i}}(x) \in [0, 1]$. This exact characterization eliminates any dependence on the reference policy's minimum probability mass, yielding the uniform bound $\left\| V_i^{\dagger, \nu_{-i}} \right\|_\infty \le 1$.
\end{proof}

To invoke our Reference-Anchored Unilateral Concentrability (Assumption~\ref{ass:concentrability}), we shift the evaluation distribution of the opponents from the learned joint policy $\hat{\pi}_{-i}$ to the fixed reference policy $\pi_{-i}^{\mathrm{ref}}$, while maintaining the arbitrary deviation policy $\pi_i'$ for Player $i$. 

To bound the distribution shift, we utilize the first-order optimality of the regularized Nash Equilibrium. For each opponent $j \neq i$, the marginalized policy $\hat{\pi}_j$ is the unique maximizer of the regularized payoff, which implies the Gibbs form:
\begin{align}
    \hat{\pi}_j(a_j \mid x) = \frac{\pi_j^{\mathrm{ref}}(a_j \mid x) \exp\big(\eta \bar{Q}_{j}(x, a_j)\big)}{\sum_{a_j'} \pi_j^{\mathrm{ref}}(a_j' \mid x) \exp\big(\eta \bar{Q}_{j}(x, a_j')\big)},
\end{align}
where $\bar{Q}_{j}$ is the marginalized empirical Q-function. Since function class $\mathcal{Q}_j$ consists of functions mapping to $[0, 1]$, $\|\bar{Q}_{j}\|_\infty \le 1$. We obtain the point-wise ratio for each action $a_j$:
\begin{align}
    \frac{\hat{\pi}_j(a_j \mid x)}{\pi_j^{\mathrm{ref}}(a_j \mid x)} = \frac{\exp\big(\eta \bar{Q}_{j}(x, a_j)\big)}{\mathbb{E}_{a_j' \sim \pi_j^{\mathrm{ref}}} \exp\big(\eta \bar{Q}_{j}(x, a_j')\big)} \le \frac{\exp(\eta)}{1} = \exp(\eta).
\end{align}
Since the joint policy of the $m-1$ opponents factorizes at each state, the joint density ratio is bounded by $\exp(\eta(m-1))$. To unify our notation across proofs, we define a single global distribution shift constant bounding the density ratio of all $m$ players:
\begin{align}
    \Lambda_{\eta, m} \coloneqq \exp(\eta m).
\end{align}
Thus, the total action likelihood ratio for the $m-1$ opponents is strictly bounded by $\Lambda_{\eta, m}$:
\begin{align}
    \frac{\hat{\pi}_{-i}(\boldsymbol{a}_{-i} \mid x)}{\pi_{-i}^{\mathrm{ref}}(\boldsymbol{a}_{-i} \mid x)} = \prod_{j \neq i} \frac{\hat{\pi}_j(a_j \mid x)}{\pi_j^{\mathrm{ref}}(a_j \mid x)} \le \exp\big(\eta (m-1)\big) \le \Lambda_{\eta, m}.
\end{align}

To relate the expectations over different context-action distributions, we consider the likelihood ratio of a joint action $\boldsymbol{a}$ under the algorithm's deviation policy $(\pi_i', \hat{\pi}_{-i})$ versus the reference deviation policy $(\pi_i', \pi_{-i}^{\mathrm{ref}})$. Because the environment distribution $\rho$ is identical for both, they cancel out in the ratio, leaving only the product of the marginalized opponent policy ratios. 

Using this bound to perform a change of measure, we shift the evaluation from the algorithm's trajectory to the covered unilateral deviation trajectory:
\begin{align}
    \mathbb{E}_{x \sim \rho} \left[ \max_{a_i} \mathbb{E}_{\boldsymbol{a}_{-i} \sim \hat{\pi}_{-i}} \big[ \mathcal{Z}_i(x, \boldsymbol{a})^2 \big] \right] \le \Lambda_{\eta, m} \cdot \mathbb{E}_{x \sim \rho} \left[ \max_{a_i} \mathbb{E}_{\boldsymbol{a}_{-i} \sim \pi_{-i}^{\mathrm{ref}}} \big[ \mathcal{Z}_i(x, \boldsymbol{a})^2 \big] \right].
\end{align}

To formalize the application of our data coverage assumption on this worst-case action, we explicitly construct the deviation policy that maximizes the expected squared error. Let $\tilde{\pi}_i$ be the deterministic policy that greedily selects the error-maximizing action for Player $i$ at each context: $\tilde{\pi}_i(x) \coloneqq \argmax_{a_i} \mathbb{E}_{\boldsymbol{a}_{-i} \sim \pi_{-i}^{\mathrm{ref}}} [\mathcal{Z}_i(x, a_i, \boldsymbol{a}_{-i})^2]$. The resulting joint policy profile $\tilde{\pi} = (\tilde{\pi}_i, \pi_{-i}^{\mathrm{ref}})$ strictly belongs to the set of reference-anchored unilateral deviations $\Pi_{\mathrm{ref-uni}}$. 

By substituting this maximizing policy, we convert the worst-case action into an expectation over the joint distribution $\rho(x)\tilde{\pi}(\boldsymbol{a}|x)$. Applying Assumption~\ref{ass:concentrability}, we shift the evaluation to the offline dataset distribution $\mu$:
\begin{align}
    \mathbb{E}_{x \sim \rho} \left[ \max_{a_i} \mathbb{E}_{\boldsymbol{a}_{-i} \sim \pi_{-i}^{\mathrm{ref}}} \big[ \mathcal{Z}_i(x, \boldsymbol{a})^2 \big] \right] 
    &= \mathbb{E}_{x \sim \rho, \boldsymbol{a} \sim \tilde{\pi}(\cdot|x)} \big[ \mathcal{Z}_i(x, \boldsymbol{a})^2 \big] \nonumber \\
    &= \mathbb{E}_{(x, \boldsymbol{a}) \sim \mu} \left[ \frac{\rho(x)\tilde{\pi}(\boldsymbol{a}|x)}{\mu(x, \boldsymbol{a})} \mathcal{Z}_i(x, \boldsymbol{a})^2 \right] \nonumber \\
    &\le C_{\mathrm{uni}} \mathbb{E}_{\mu} \big[ \mathcal{Z}_i(x, \boldsymbol{a})^2 \big].
\end{align}

Combining this with the opponent distribution shift $\Lambda_{\eta, m}$, together with the previous bound we derived in \eqref{Term 2 3 in Z unshift}, we obtain the fully shifted bound:
\begin{align}
    \mathbb{E}_{x \sim \rho} \big[ \text{Term II}_i + \text{Term III}_i \big] \le \frac{\eta}{2} \Lambda_{\eta, m} C_{\mathrm{uni}} \mathbb{E}_{\mu} \big[ \mathcal{Z}_i^2 \big].
\end{align}

Using the standard fast-rate statistical guarantee for regularized least-squares regression, the expected in-sample squared regression error under the data distribution $\mu$ is bounded by the function class complexity $|\mathcal{Q}_i|$ and the number of samples $n$ with high probability:
\begin{align}
    \mathbb{E}_{\mu} \big[ \mathcal{Z}_i^2 \big] \le \mathcal{O}\left( \frac{\log(|\mathcal{Q}_i|/\delta)}{n} \right).
\end{align}

By substituting this rate into our unrolled bound, we effectively resolve the expectation over $\mu$ by replacing it with its statistical upper limit. We conclude the final suboptimality gap for Player $i$:
\begin{align}
    \mathbb{E}_{x \sim \rho} \big[ \text{Term II}_i + \text{Term III}_i \big] \le \widetilde{\mathcal{O}}\left( \frac{\eta \Lambda_{\eta, m} C_{\mathrm{uni}} \log |\mathcal{Q}_i|}{n} \right).
\end{align}

\paragraph{Final Bound on Total Exploitability.}
The Total Exploitability of the game is defined as the sum of the unilateral suboptimality across all $m$ players. By summing the bound derived for Player $i$ over the entire set $\mathcal{N} = [m]$, and absorbing the distribution shift constant $\Lambda_{\eta, m} = \exp(\eta m)$ and the logarithmic terms into the $\widetilde{\mathcal{O}}$ notation, we obtain:
\begin{align}
    \mathrm{Gap}_{\mathrm{NE}}(\hat{\pi}) &= \sum_{i=1}^m \mathbb{E}_{x \sim \rho} \left[ V_i^{\dagger, \hat{\pi}_{-i}}(x) - V_i^{\hat{\pi}}(x) \right] \nonumber \\
    &\le \sum_{i=1}^m \left( \text{Term II}_i + \text{Term III}_i \right)\\
    &\le \widetilde{\mathcal{O}}\left( \frac{m \eta \Lambda_{\eta, m} C_{\mathrm{uni}} \log |\mathcal{Q}|}{n} \right).
\end{align}
Substituting the bound $\Lambda_{\eta, m} = \exp(\eta m)$, the final statistical rate is:
\begin{align}
    \mathrm{Gap}_{\mathrm{NE}}(\hat{\pi}) \le \widetilde{\mathcal{O}}\left( \frac{m \eta C_{\mathrm{uni}} \exp(\eta m) \log |\mathcal{Q}|}{n} \right).
\end{align}
This result establishes that Algorithm 1 (GANE) achieves a fast statistical rate of $\widetilde{\mathcal{O}}(1/n)$ in the $m$-player general-sum offline setting. By operating in the Contextual Bandit setting, we completely eliminate the exponential horizon dependencies ($\exp(H)$) that artificially inflate pessimism-free RL bounds, explicitly highlighting the optimal statistical efficiency of product-policy equilibrium recovery.

\subsection{Proof of Lemma~\ref{lem:omd_external_regret}}
\label{app:lemma omd proof}
\begin{proof}
We drop player and context subscripts throughout and let $\pi^* \in \argmax_{\pi \in \Delta(\mathcal{A})} \sum_{t=1}^T f^{(t)}(\pi)$ denote an arbitrary comparator. Denote the cumulative objective by $F_t(\pi) \coloneqq \sum_{k=1}^t f^{(k)}(\pi)$, with $F_0 \equiv 0$.

\paragraph{Iterates as Follow-The-Leader.} Unrolling the geometric-mean update in Algorithm~\ref{alg:pro_md} with initialization $\pi^{(1)} = \pi^{\mathrm{ref}}$ and decaying stepsize $\gamma_t = 1/t$ yields the closed form $\pi^{(t+1)}(\cdot) \propto \pi^{\mathrm{ref}}(\cdot) \cdot \exp\!\big( \tfrac{\eta}{t} \sum_{k=1}^t Q^{(k)}(\cdot) \big)$, which is precisely the Follow-The-Leader (FTL) iterate on the cumulative objective,
\begin{align}
    \pi^{(t+1)} = \argmax_{\pi \in \Delta(\mathcal{A})} F_t(\pi). \label{eq:ftl_rep}
\end{align}

\paragraph{Be-The-Leader.} By induction on $T$, we show
\begin{align}
    \sum_{t=1}^T f^{(t)}(\pi^{(t+1)}) \ge \sum_{t=1}^T f^{(t)}(\pi^*). \label{eq:btl_main}
\end{align}
The base case $T = 1$ is immediate since $\pi^{(2)} = \argmax f^{(1)}$. For the inductive step, apply the hypothesis at $T-1$ with comparator $\pi^{(T+1)}$:
\begin{align}
    \sum_{t=1}^{T} f^{(t)}(\pi^{(t+1)}) = \sum_{t=1}^{T-1} f^{(t)}(\pi^{(t+1)}) + f^{(T)}(\pi^{(T+1)}) \ge F_{T-1}(\pi^{(T+1)}) + f^{(T)}(\pi^{(T+1)}) = F_T(\pi^{(T+1)}) \ge F_T(\pi^*),
\end{align}
where the final inequality uses $\pi^{(T+1)} = \argmax F_T$. Rearranging~\eqref{eq:btl_main} bounds the regret by the iterate stability:
\begin{align}
    \sum_{t=1}^T \left( f^{(t)}(\pi^*) - f^{(t)}(\pi^{(t)}) \right) \le \sum_{t=1}^T \left( f^{(t)}(\pi^{(t+1)}) - f^{(t)}(\pi^{(t)}) \right). \label{eq:btl_stability}
\end{align}

\paragraph{Strong-concavity lower bound.} Pinsker's inequality $\mathrm{KL}(p \| q) \ge \tfrac{1}{2} \|p - q\|_1^2$ shows that the negative entropy is $1$-strongly convex with respect to $\|\cdot\|_1$, so each $f^{(t)}$ is $\eta^{-1}$-strongly concave and $F_t$ is $(t/\eta)$-strongly concave. Applying strong concavity at the maximizers $\pi^{(t)} = \argmax F_{t-1}$ and $\pi^{(t+1)} = \argmax F_t$ gives
\begin{align}
    F_{t-1}(\pi^{(t)}) - F_{t-1}(\pi^{(t+1)}) &\ge \tfrac{t-1}{2\eta} \|\pi^{(t+1)} - \pi^{(t)}\|_1^2, \\
    F_t(\pi^{(t+1)}) - F_t(\pi^{(t)}) &\ge \tfrac{t}{2\eta} \|\pi^{(t+1)} - \pi^{(t)}\|_1^2.
\end{align}
Summing the two inequalities and using $F_t - F_{t-1} = f^{(t)}$ yields the per-step stability lower bound
\begin{align}
    f^{(t)}(\pi^{(t+1)}) - f^{(t)}(\pi^{(t)}) \ge \tfrac{2t-1}{2\eta} \|\pi^{(t+1)} - \pi^{(t)}\|_1^2. \label{eq:sc_lower}
\end{align}

\paragraph{H\"older upper bound.} The Gibbs form~\eqref{eq:ftl_rep} gives $\log(\pi^{(t)}/\pi^{\mathrm{ref}}) = \tfrac{\eta}{t-1} \sum_{k=1}^{t-1} Q^{(k)} - \log Z_{t-1}$ for $t \ge 2$, and $\log(\pi^{(1)}/\pi^{\mathrm{ref}}) = 0$ for $t = 1$. Substituting into $\nabla f^{(t)}(\pi) = Q^{(t)} - \eta^{-1}(\log(\pi/\pi^{\mathrm{ref}}) + \mathbf{1})$ and projecting onto the simplex tangent space (i.e., discarding components parallel to $\mathbf{1}$, which vanish against $\pi^{(t+1)} - \pi^{(t)}$ since $\mathbf{1}^\top(\pi^{(t+1)} - \pi^{(t)}) = 0$),
\begin{align}
    \langle \nabla f^{(t)}(\pi^{(t)}), \pi^{(t+1)} - \pi^{(t)} \rangle = \Big\langle Q^{(t)} - \tfrac{1}{t-1} \textstyle\sum_{k=1}^{t-1} Q^{(k)},\, \pi^{(t+1)} - \pi^{(t)} \Big\rangle,
\end{align}
with the convention that the second term vanishes when $t=1$. Because each $Q^{(k)} \in [0,1]^{|\mathcal{A}|}$, the centered gradient has $\ell_\infty$-norm at most $1$. Concavity of $f^{(t)}$ combined with H\"older's inequality therefore gives
\begin{align}
    f^{(t)}(\pi^{(t+1)}) - f^{(t)}(\pi^{(t)}) \le \langle \nabla f^{(t)}(\pi^{(t)}), \pi^{(t+1)} - \pi^{(t)} \rangle \le \|\pi^{(t+1)} - \pi^{(t)}\|_1. \label{eq:holder_upper}
\end{align}

\paragraph{Combining the bounds.} Equations~\eqref{eq:sc_lower} and~\eqref{eq:holder_upper} imply $\|\pi^{(t+1)} - \pi^{(t)}\|_1 \le \tfrac{2\eta}{2t-1}$, and substituting back into~\eqref{eq:holder_upper} gives the per-step stability bound $f^{(t)}(\pi^{(t+1)}) - f^{(t)}(\pi^{(t)}) \le \tfrac{2\eta}{2t-1}$. Summing and applying~\eqref{eq:btl_stability},
\begin{align}
    \sum_{t=1}^T \left( f^{(t)}(\pi^*) - f^{(t)}(\pi^{(t)}) \right) \le \sum_{t=1}^T \frac{2\eta}{2t-1} \le 2\eta (1 + \log T).
\end{align}
Dividing by $T$ yields the claimed $\mathcal{O}\big( \tfrac{\eta \log T}{T} \big)$ bound.
\end{proof}

\subsection{Proof of Theorem~\ref{thm:gamd_cce_rate}}
\label{app:proof_pro_md_cce}

In this section, we provide the full sample complexity proof for General-sum Anchored Mirror Descent (GAMD, Algorithm \ref{alg:pro_md}). The algorithm outputs a time-averaged joint policy $\bar{\pi}$, where $\bar{\pi}(\boldsymbol{a}|x) = \frac{1}{T} \sum_{t=1}^T \prod_{i=1}^m \pi_i^{(t)}(a_i|x)$ is a correlated joint distribution representing an approximate Coarse Correlated Equilibrium (CCE).

We evaluate the convergence of the algorithm using the Total Exploitability Gap. For the learned policy $\bar{\pi}$, the expected gap is defined as the sum of the unilateral improvements available to all $m$ players:
\begin{align}
    \mathrm{Gap}_{\mathrm{CCE}}(\bar{\pi}) := \sum_{i=1}^m \mathbb{E}_{x \sim \rho} \left[ V_i^{\dagger, \bar{\pi}_{-i}}(x) - V_i^{\bar{\pi}}(x) \right].
\end{align}

\subsubsection*{Step 1: Gap Decomposition via Estimated Values}
To decouple the optimization error introduced by running a finite number of Mirror Descent steps $T$ from the statistical error caused by finite offline samples $n$, we pivot through the estimated value functions $\hat{V}$ and $\hat{Q}$ constructed by the algorithm. 

For each player $i \in [m]$, we decompose their unilateral exploitability gap identically to Lemma~\ref{lem:gap_decomp_main}:
\begin{align}
    V_i^{\dagger, \bar{\pi}_{-i}}(x) - V_i^{\bar{\pi}}(x) &= \underbrace{\left( \hat{V}_i^{\dagger, \bar{\pi}_{-i}}(x) - \hat{V}_i^{\bar{\pi}}(x) \right)}_{\text{Term I: Empirical Optimization Gap}} \nonumber \\
    &\quad + \underbrace{\left( V_i^{\dagger, \bar{\pi}_{-i}}(x) - \hat{V}_i^{\dagger, \bar{\pi}_{-i}}(x) \right)}_{\text{Term II: Evaluation Error of Best Response}} \nonumber \\
    &\quad + \underbrace{\left( \hat{V}_i^{\bar{\pi}}(x) - V_i^{\bar{\pi}}(x) \right)}_{\text{Term III: Evaluation Error of Joint Policy}}.
\end{align}

\subsubsection*{Step 2: Bounding the Optimization Error (Term I)}
Term I represents the empirical CCE gap of the time-averaged policy $\bar{\pi}$ on the estimated bandit defined by $\hat{Q}$. We formally bound this by connecting the exploitability of the average joint policy to the external regret of the independent Mirror Descent iterates.

For any context $x$, we evaluate Player $i$'s empirical suboptimality against the time-averaged opponents $\bar{\pi}_{-i}$. For notational clarity, we define the marginalized empirical Q-function as $\bar{Q}_i(x, a_i) \coloneqq \mathbb{E}_{\boldsymbol{a}_{-i} \sim \bar{\pi}_{-i}}[\hat{Q}_i(x, a_i, \boldsymbol{a}_{-i})]$. Because the estimated Q-function is linear in the opponents' distribution, we have exact equality for the expected payoff:
\begin{align}
    \bar{Q}_i(x, a_i) = \frac{1}{T} \sum_{t=1}^T \bar{Q}_i^{(t)}(x, a_i).
\end{align}
Furthermore, by the convexity of the KL-divergence, the regularizer of the average policy is bounded by the average of the regularizers:
\begin{align}
    \mathrm{KL}(\bar{\pi}_i(\cdot|x) \,\|\, \pi_i^{\mathrm{ref}}(\cdot|x)) \le \frac{1}{T} \sum_{t=1}^T \mathrm{KL}(\pi_i^{(t)}(\cdot|x) \,\|\, \pi_i^{\mathrm{ref}}(\cdot|x)).
\end{align}

Applying these two properties, the empirical CCE gap for Player $i$ at context $x$ is strictly bounded by their average external regret against the historical sequence of policies:
\begin{align}
    &\max_{\pi_i} \left[ \bar{Q}_i(x, \pi_i) - \eta^{-1} \mathrm{KL}(\pi_i \| \pi_i^{\mathrm{ref}}) \right] - \left[ \bar{Q}_i(x, \bar{\pi}_i) - \eta^{-1} \mathrm{KL}(\bar{\pi}_i \| \pi_i^{\mathrm{ref}}) \right] \nonumber \\
    &\le \max_{\pi_i \in \Delta(\mathcal{A}_i)} \frac{1}{T} \sum_{t=1}^T \underbrace{\left( \bar{Q}_i^{(t)}(x, \pi_i) - \eta^{-1} \mathrm{KL}(\pi_i \| \pi_i^{\mathrm{ref}}) \right)}_{\coloneqq f_{i,x}^{(t)}(\pi_i)} - \frac{1}{T} \sum_{t=1}^T \underbrace{\left( \bar{Q}_i^{(t)}(x, \pi_i^{(t)}) - \eta^{-1} \mathrm{KL}(\pi_i^{(t)} \| \pi_i^{\mathrm{ref}}) \right)}_{= f_{i,x}^{(t)}(\pi_i^{(t)})}.
\end{align}

At each iteration $t$, Algorithm~\ref{alg:pro_md} independently updates the policies $\pi_i^{(t)}$ via KL-regularized Online Mirror Descent (OMD) on the contextual objectives $f_{i,x}^{(t)}$. To bound the external regret of this update sequence, we rely on foundational results from the Online Convex Optimization (OCO) literature. While standard online learning algorithms typically suffer an $\tilde{\mathcal{O}}(\sqrt{T})$ cumulative regret bound, the explicit presence of the KL-divergence penalty in our formulation strictly transforms the objective into a strongly concave function (when maximizing payoff). As established in classical OCO texts, applying no-regret learning to a strongly concave objective with the decaying stepsize schedule $\gamma_t = 1/t$ matching the update in Algorithm~\ref{alg:pro_md} attains a cumulative regret of $\tilde{\mathcal{O}}(\log T)$. 



By the standard online-to-offline reduction for normal-form games, the total empirical optimization gap is bounded by the expected context-wise average regret. Using Lemma~\ref{lem:omd_external_regret}:
\begin{align}
    \mathbb{E}_{x \sim \rho} [\text{Term I}_i] &\le \mathbb{E}_{x \sim \rho} \left[ \max_{\pi_i} \frac{1}{T} \sum_{t=1}^T \left( f_{i,x}^{(t)}(\pi_i) - f_{i,x}^{(t)}(\pi_i^{(t)}) \right) \right] \nonumber \\
    &\le \mathcal{O}\left( \frac{\eta \log T}{T} \right) = \widetilde{\mathcal{O}}\left( \frac{\eta}{T} \right).
\end{align}
Summing this optimization error over all $m$ players yields a total empirical optimization gap that strictly decays at a fast rate of $\widetilde{\mathcal{O}}(1/T)$.

\subsubsection*{Step 3: Unrolling the Statistical Errors (Terms II and III)}
We now bound the statistical evaluation errors (Term II and Term III). We introduce the regression error for Player $i$ evaluated under the algorithm's estimated Q-functions:
\begin{align}
    \mathcal{Z}_i(x, \boldsymbol{a}) \coloneqq \hat{Q}_i(x, \boldsymbol{a}) - r^\star_i(x, \boldsymbol{a}).
\end{align}

\paragraph{Evaluating Term II (Best Response Error).}
For Term II, we evaluate the suboptimality of the estimated best response. Let $\Delta_i^\dagger(x) \coloneqq V_i^{\dagger, \bar{\pi}_{-i}}(x) - \hat{V}_i^{\dagger, \bar{\pi}_{-i}}(x)$. By applying the 1-smoothness of the log-partition function exactly as derived in the idealized proof (Equation \ref{Delta 1 smooth Q}), we obtain the local bound:
\begin{align}
    \Delta_i^\dagger(x) \le \mathbb{E}_{a_i \sim \bar{\pi}_i^\dagger(\cdot|x)} \left[ Q_i^{\dagger, \bar{\pi}_{-i}}(x, a_i) - \bar{Q}_i(x, a_i) \right] + \frac{\eta}{2} \left\| Q_i^{\dagger, \bar{\pi}_{-i}}(x, \cdot) - \bar{Q}_i(x, \cdot) \right\|_\infty^2,
\end{align}
where $\bar{\pi}_i^\dagger$ is the empirical best-response policy. 

By direct algebraic substitution, the linear Q-value difference is strictly equal to the negative regression error:
\begin{align}
    Q_i^{\dagger, \bar{\pi}_{-i}}(x, a_i) - \bar{Q}_i(x, a_i) = \mathbb{E}_{\boldsymbol{a}_{-i} \sim \bar{\pi}_{-i}} \left[ -\mathcal{Z}_i(x, \boldsymbol{a}) \right].
\end{align}
Taking the expectation over the context distribution $x \sim \rho$, Term II is bounded by:
\begin{align}
    \mathbb{E}_{x \sim \rho}[\text{Term II}_i(x)] \le \mathbb{E}_{x \sim \rho, \boldsymbol{a} \sim (\bar{\pi}_i^\dagger, \bar{\pi}_{-i})} \left[ -\mathcal{Z}_i(x, \boldsymbol{a}) \right] + \frac{\eta}{2} \mathbb{E}_{x \sim \rho} \left[ \left\| Q_i^{\dagger, \bar{\pi}_{-i}}(x, \cdot) - \bar{Q}_i(x, \cdot) \right\|_\infty^2 \right].
\end{align}

\paragraph{Evaluating Term III (Joint Policy Error).}
Term III measures the evaluation error of the algorithm's joint policy: $\hat{V}_i^{\bar{\pi}}(x) - V_i^{\bar{\pi}}(x)$. Because the identical policy $\bar{\pi}$ is used for both the empirical and true values, the KL-regularization terms exactly cancel out. Direct substitution yields:
\begin{align}
    \mathbb{E}_{x \sim \rho}[\text{Term III}_i(x)] = \mathbb{E}_{x \sim \rho, \boldsymbol{a} \sim \bar{\pi}} \left[ \mathcal{Z}_i(x, \boldsymbol{a}) \right].
\end{align}

\subsubsection*{Step 4: The Distribution Mismatch and Absolute Errors}
In the analysis of exact Nash Equilibria, the joint policy is a pure product distribution ($\hat{\pi} = \hat{\pi}_i \times \hat{\pi}_{-i}$) and the learned policy is identically the best response ($\hat{\pi}_i = \hat{\pi}_i^\dagger$). This structural alignment causes the evaluation distributions to match perfectly, allowing the linear errors $-\mathcal{Z}_i$ and $+\mathcal{Z}_i$ to cancel entirely. 

However, the time-averaged policy $\bar{\pi}$ produced by GAMD is a \textit{correlated} joint distribution. Consequently, the independent product of the best response and the marginalized opponents is fundamentally mismatched from the actual correlated joint policy:
\begin{align}
    \bar{\pi}_i^\dagger(a_i|x) \times \bar{\pi}_{-i}(\boldsymbol{a}_{-i}|x) \neq \bar{\pi}(\boldsymbol{a}|x).
\end{align}
Because the expectations for the linear regression errors are taken over different distributions, they do not cancel. To establish a rigorous upper bound, we must bound their absolute values. By the triangle inequality, we maintain the exact squared error term:
\begin{align}
    \text{Term II}_i(x) + \text{Term III}_i(x) &\le \mathbb{E}_{\boldsymbol{a} \sim (\bar{\pi}_i^\dagger, \bar{\pi}_{-i})} \left[ |\mathcal{Z}_i(x, \boldsymbol{a})| \right] + \mathbb{E}_{\boldsymbol{a} \sim \bar{\pi}} \left[ |\mathcal{Z}_i(x, \boldsymbol{a})| \right] \nonumber \\
    &\quad + \frac{\eta}{2} \left\| Q_i^{\dagger, \bar{\pi}_{-i}}(x, \cdot) - \bar{Q}_i(x, \cdot) \right\|_\infty^2.
\end{align}

\subsubsection*{Step 5: Bounding the $L_\infty$ Error and Applying Cauchy-Schwarz}
To bridge the pointwise absolute errors to our least-squares statistical oracle, we apply Cauchy-Schwarz ($\mathbb{E}[|X|] \le \sqrt{\mathbb{E}[X^2]}$). 

For the squared Q-value difference, we must unroll the $L_\infty$ norm into the pointwise squared regression errors. Taking the absolute value of the Q-value difference identity yields:
\begin{align}
    \left\| Q_i^{\dagger, \bar{\pi}_{-i}}(x, \cdot) - \bar{Q}_i(x, \cdot) \right\|_\infty \le \max_{a_i} \mathbb{E}_{\boldsymbol{a}_{-i} \sim \bar{\pi}_{-i}} \Big[ |\mathcal{Z}_i(x, \boldsymbol{a})| \Big].
\end{align}
Squaring both sides and applying Jensen's inequality for the inner expectation bounds the $L_\infty$ norm strictly by the expected squared regression error under a maximally adverse unilateral deviation:
\begin{align}
    \left\| Q_i^{\dagger, \bar{\pi}_{-i}}(x, \cdot) - \bar{Q}_i(x, \cdot) \right\|_\infty^2 \le \max_{a_i} \mathbb{E}_{\boldsymbol{a}_{-i} \sim \bar{\pi}_{-i}} \big[ \mathcal{Z}_i(x, \boldsymbol{a})^2 \big].
\end{align}

Summing the expectations over all players and applying Cauchy-Schwarz to the linear terms, we isolate the statistical gap strictly in terms of squared errors:
\begin{align}
    &\sum_{i=1}^m \mathbb{E}_{x \sim \rho}[\text{Term II}_i(x) + \text{Term III}_i(x)] \nonumber \\
    &\le \sum_{i=1}^m \left( \sqrt{ \mathbb{E}_{x \sim \rho, \boldsymbol{a} \sim (\bar{\pi}_i^\dagger, \bar{\pi}_{-i})} \left[ \mathcal{Z}_i^2 \right] } + \sqrt{ \mathbb{E}_{x \sim \rho, \boldsymbol{a} \sim \bar{\pi}} \left[ \mathcal{Z}_i^2 \right] } \right) \nonumber \\
    &\quad + \frac{\eta}{2} \sum_{i=1}^m \mathbb{E}_{x \sim \rho} \left[ \max_{a_i} \mathbb{E}_{\boldsymbol{a}_{-i} \sim \bar{\pi}_{-i}} \big[ \mathcal{Z}_i(x, \boldsymbol{a})^2 \big] \right].
\end{align}

\subsubsection*{Step 6: Distribution Shift and Data Coverage}
We bound each expected term by shifting the evaluation distribution to the offline dataset distribution $\mu$ via the Reference-Anchored Unilateral Concentrability (Assumption~\ref{ass:concentrability}). 

Crucially, because GAMD explicitly constructs the correlated policy via exponentiated Mirror Descent updates, we can strictly bound the density ratio of every individual iterate $\pi_{i}^{(t)}$ relative to the reference policy. Initialized at $\pi_i^{(1)} = \pi_i^{\mathrm{ref}}$, the decaying-stepsize closed-form update (Step~\ref{eq:md_update_closed_form}, Algorithm~\ref{alg:pro_md}) unrolls to the Gibbs form anchored to the reference policy:
\begin{align}
    \pi_i^{(t+1)}(a_i|x) \propto \pi_i^{\mathrm{ref}}(a_i|x) \exp\big(\eta \tilde{Q}_i^{(t)}(x, a_i)\big), \quad \tilde{Q}_i^{(t)}(x, a_i) \coloneqq \frac{1}{t} \sum_{k=1}^{t} \bar{Q}_i^{(k)}(x, a_i).
\end{align}
Since the function class $\mathcal{Q}_j$ consists of functions mapping to $[0, 1]$, we have $\|\bar{Q}_i^{(k)}\|_\infty \le 1$, so the arithmetic average $\tilde{Q}_i^{(t)} \in [0, 1]$. Evaluating the normalization constant yields the exact pointwise bound for any action:
\begin{align}
    \frac{\pi_i^{(t+1)}(a_i|x)}{\pi_i^{\mathrm{ref}}(a_i|x)} = \frac{\exp\big(\eta \tilde{Q}_i^{(t)}(x, a_i)\big)}{\mathbb{E}_{a' \sim \pi_i^{\mathrm{ref}}} \exp\big(\eta \tilde{Q}_i^{(t)}(x, a')\big)} \le \frac{\exp(\eta)}{1} = \exp(\eta).
\end{align}

Because the marginalized opponent policy $\bar{\pi}_{-i}$ and the joint policy $\bar{\pi}$ are convex combinations of these iterates, they rigorously preserve this upper bound. To unify our notation across the $m$ players, we define a single global distribution shift constant:
\begin{align}
    \Lambda_{\eta, m} \coloneqq \exp(\eta m).
\end{align}

For the unilateral deviation trajectories, we encounter two distinct expectations: one over the empirical best response profile $(\bar{\pi}_i^\dagger, \bar{\pi}_{-i})$ and one involving a worst-case maximization $\max_{a_i} \mathbb{E}_{\boldsymbol{a}_{-i} \sim \bar{\pi}_{-i}}[\cdot]$. In both cases, shifting the $m-1$ opponents from $\bar{\pi}_{-i}$ to the reference policy $\pi_{-i}^{\mathrm{ref}}$ produces a joint density ratio bounded by $\exp(\eta(m-1)) \le \Lambda_{\eta, m}$. 

For the first term, the profile $(\bar{\pi}_i^\dagger, \pi_{-i}^{\mathrm{ref}})$ strictly belongs to $\Pi_{\mathrm{ref-uni}}$, allowing us to directly apply the concentrability coefficient $C_{\mathrm{uni}}$:
\begin{align}
    \mathbb{E}_{x \sim \rho, \boldsymbol{a} \sim (\bar{\pi}_i^\dagger, \bar{\pi}_{-i})} \left[ \mathcal{Z}_i^2 \right] \le \Lambda_{\eta, m} \mathbb{E}_{x \sim \rho, \boldsymbol{a} \sim (\bar{\pi}_i^\dagger, \pi_{-i}^{\mathrm{ref}})} \left[ \mathcal{Z}_i^2 \right] \le \Lambda_{\eta, m} C_{\mathrm{uni}} \mathbb{E}_{\mu} \left[ \mathcal{Z}_i^2 \right].
\end{align}

For the second term, to formalize the application of our data coverage assumption on the maximization, we construct a deterministic greedy policy $\tilde{\pi}_i(x) \coloneqq \argmax_{a_i} \mathbb{E}_{\boldsymbol{a}_{-i} \sim \pi_{-i}^{\mathrm{ref}}} [\mathcal{Z}_i(x, a_i, \boldsymbol{a}_{-i})^2]$. The resulting joint profile $\tilde{\pi} = (\tilde{\pi}_i, \pi_{-i}^{\mathrm{ref}})$ also strictly belongs to $\Pi_{\mathrm{ref-uni}}$, yielding:
\begin{align}
    \mathbb{E}_{x \sim \rho} \left[ \max_{a_i} \mathbb{E}_{\boldsymbol{a}_{-i} \sim \bar{\pi}_{-i}} \big[ \mathcal{Z}_i(x, \boldsymbol{a})^2 \big] \right] 
    &\le \Lambda_{\eta, m} \mathbb{E}_{x \sim \rho, \boldsymbol{a} \sim \tilde{\pi}(\cdot|x)} \big[ \mathcal{Z}_i(x, \boldsymbol{a})^2 \big] \nonumber \\
    &= \Lambda_{\eta, m} \mathbb{E}_{(x, \boldsymbol{a}) \sim \mu} \left[ \frac{\rho(x)\tilde{\pi}(\boldsymbol{a}|x)}{\mu(x, \boldsymbol{a})} \mathcal{Z}_i(x, \boldsymbol{a})^2 \right] \nonumber \\
    &\le \Lambda_{\eta, m} C_{\mathrm{uni}} \mathbb{E}_{\mu} \big[ \mathcal{Z}_i^2 \big].
\end{align}

Similarly, for the on-policy evaluation over the correlated joint policy $\bar{\pi}$, we shift the actions of all $m$ players to the joint reference policy $\pi^{\mathrm{ref}} = \prod_{j=1}^m \pi_j^{\mathrm{ref}}$. The total density ratio is strictly bounded by $\prod_{j=1}^m \exp(\eta) = \exp(\eta m) = \Lambda_{\eta, m}$. Because the joint reference policy $\pi^{\mathrm{ref}}$ is explicitly covered by the unilateral assumption set $\Pi_{\mathrm{ref-uni}}$ (representing the trivial case where the deviation policy equals the reference policy), we directly apply $C_{\mathrm{uni}}$ to bound the joint evaluation:
\begin{align}
    \mathbb{E}_{x \sim \rho, \boldsymbol{a} \sim \bar{\pi}} \left[ \mathcal{Z}_i^2 \right] \le \Lambda_{\eta, m} C_{\mathrm{uni}} \mathbb{E}_{\mu} \left[ \mathcal{Z}_i^2 \right].
\end{align}

By applying the fast-rate guarantee for regularized least-squares regression, the in-sample expected squared regression error is bounded by $\mathcal{O}\left( \frac{\log |\mathcal{Q}_i|}{n} \right)$. Substituting this statistical rate into the square-root and linear components yields:
\begin{align}
    &\sum_{i=1}^m \mathbb{E}_{x \sim \rho}[\text{Term II}_i(x) + \text{Term III}_i(x)] \nonumber \\
    &\le \sum_{i=1}^m 2 \sqrt{ \Lambda_{\eta, m} C_{\mathrm{uni}} \mathcal{O}\left( \frac{\log |\mathcal{Q}_i|}{n} \right) } + \sum_{i=1}^m \mathcal{O}\left( \frac{\eta \Lambda_{\eta, m} C_{\mathrm{uni}} \log |\mathcal{Q}_i|}{n} \right).
\end{align}

\subsubsection*{Step 7: Final Bound Synthesis}
Summing the empirical optimization error from Step 2 and the combined statistical evaluation error from Step 6, we obtain the final bound on the expected Total Exploitability Gap. Because the absolute linear errors decay at $\tilde{\mathcal{O}}(1/\sqrt{n})$, they asymptotically dominate the $\tilde{\mathcal{O}}(1/n)$ squared Q-difference term. 
\begin{align}
    \mathbb{E} [\mathrm{Gap}_{\mathrm{CCE}}(\bar{\pi})] &= \sum_{i=1}^m \mathbb{E}_{x \sim \rho} \left[ \text{Term I}_i(x) + \text{Term II}_i(x) + \text{Term III}_i(x) \right] \nonumber \\
    &\le \widetilde{\mathcal{O}}\left( \frac{m \eta}{T} \right) + \widetilde{\mathcal{O}}\left( m \sqrt{ \frac{\Lambda_{\eta, m} C_{\mathrm{uni}} \log |\mathcal{Q}|}{n} } \right) + \widetilde{\mathcal{O}}\left( \frac{1}{n} \right).
\end{align}
Here, $\Lambda_{\eta, m} \coloneqq \exp(\eta m)$, and $ C_{\mathrm{uni}}$ is the coefficient for unilateral coverage. 
Setting the number of Mirror Descent iterations such that $T \ge \sqrt{n}$ ensures that the empirical optimization error is strictly dominated by the statistical evaluation error. The final sample complexity is therefore:
\begin{align}
    \mathbb{E} [\mathrm{Gap}_{\mathrm{CCE}}(\bar{\pi})] \le \widetilde{\mathcal{O}}\left( \frac{1}{\sqrt{n}} \right).
\end{align}
This result demonstrates that while GAMD attains a fast $\widetilde{\mathcal{O}}(1/T)$ empirical optimization rate, the fundamental distribution mismatch inherent to evaluating a correlated joint policy strictly bottlenecks the statistical performance at the standard minimax rate of $\widetilde{\mathcal{O}}(1/\sqrt{n})$.



\end{document}